\documentclass{article}

\usepackage[preprint]{neurips_2026}

\usepackage[utf8]{inputenc} 
\usepackage[T1]{fontenc}    
\usepackage{hyperref}       
\usepackage{url}            
\usepackage{booktabs}       
\usepackage{amsfonts}       
\usepackage{nicefrac}       
\usepackage{microtype}      
\usepackage{xcolor}         
\usepackage{enumitem}

\usepackage{amsmath}
\usepackage{amssymb}
\usepackage{amsfonts}
\usepackage{amsthm}
\usepackage{diagbox}
\usepackage{graphicx}

\newtheorem{theo}{Theorem}
\newtheorem{defi}{Definition}
\newtheorem{lemm}{Lemma}
\newtheorem{rem}{Remark}
\newtheorem{claim}{Claim}
\newtheorem{ex}{Example}
\newtheorem{prop}{Proposition}
\newtheorem{coro}{Corollary}

\title{ESENSC: A Polynomial-Time Axiomatic \\ Alternative to SHAP}

\author{%
  Kazuhiro Hiraki \\
  International Monetary Fund \\
  \texttt{khiraki@imf.org} \\
  \And
  Shinichi Ishihara \\
  Independent Researcher \\
  \texttt{ishihara5683@gmail.com} \\
  \AND
  Takumi Kongo\thanks{\url{https://sites.google.com/site/takumikongo/}} \\
  Fukuoka University \\
  \texttt{kongo@adm.fukuoka-u.ac.jp} \\
  \And
  Junnosuke Shino\thanks{\url{https://sites.google.com/site/junnosukeshino/}}  \\
  Waseda University \\
  \texttt{junnosuke.shino@waseda.jp} \\
}

\begin{document}

\maketitle

\begin{abstract}
We propose ESENSC, a computationally efficient and theoretically grounded alternative to SHAP. ESENSC is a deterministic, closed-form attribution rule that requires only linearly many coalition-value evaluations, satisfies the null-player property, and admits an axiomatic characterization based on efficiency, a restricted form of differential marginality, and explicit computational constraints.
Empirically, on tabular prediction tasks using neural network and XGBoost models, ESENSC shows close agreement with exact SHAP in both attribution magnitude and feature rankings while requiring substantially less computation. Under the evaluated settings, it compares favorably with widely used sampling-based SHAP approximations. While exact SHAP requires exponentially many coalition evaluations, ESENSC scales linearly with the number of features.
These results show that ESENSC offers a favorable trade-off between agreement with exact SHAP and computational cost, providing a practical and theoretically grounded alternative for feature attribution in high-dimensional applications.

\end{abstract}

\vspace{-5pt}
\section{Introduction}
Additive feature attribution (AFA) methods decompose a model prediction into feature-level contributions and play a central role in explainable artificial intelligence (XAI). Among them, SHAP, based on the Shapley value from cooperative game theory \cite{shapley}, has emerged as the de facto standard due to its strong theoretical guarantees and model-agnostic applicability \cite{lundberg2017}. Formally, given a trained model $f$ and an observation $x_k$, an AFA allocates the difference between the prediction $f(x_k)$ and the baseline $E[f(X)]$ across features according to their contributions. SHAP can be interpreted as the Shapley value of a cooperative game in which features are treated as players. We refer to this formulation as a XAI--TU game.

In this paper, we develop alternative attribution rules based on two classical solution concepts from cooperative game theory: equal surplus (ES) and proportional allocation (PA), and systematically compare their properties with those of SHAP.
There are two main motivations for seeking such alternatives. First, while SHAP is theoretically elegant, its exact computation requires evaluating an exponential number of coalitions ($2^n$), creating a formidable barrier in high-dimensional settings where the number of features is large. In contrast, ES and PA rules are well known in cooperative game theory as solution concepts with low computational complexity.
Second, existing approaches to addressing this computational challenge --- namely, approximation methods for SHAP --- often lack a clear theoretical foundation and may exhibit instability or limited accuracy in practice. This suggests the need for attribution methods that are both computationally efficient and theoretically well grounded.

For ES--type solutions, we construct a novel attribution rule, ESENSC, based on an equal weight combination of Equal Surplus and Egalitarian Nonseparable Contribution (ENSC) \cite{Dragan_Drissen_Funaki, Drissen_Funaki_1991}. While the naive combination is computationally attractive, it violates the null-player property. Such violations are undesirable because features with no effect on predictions should receive zero attribution. We address this issue by modifying the rule to enforce the null-player property, resulting in ESENSC, which admits a closed-form representation and can be computed in polynomial time.

For PA--type solutions, we first address the order-reversal problem, which arises when positive and negative coalition values coexist. In such cases, proportional allocation can reverse the ranking of feature contributions. We develop modified PA--type AFAs, called PARPA, that preserve the ordering of features even under sign conflicts.

We evaluate the proposed methods on tabular prediction tasks using neural networks and gradient boosting models (with additional results for random forests in the appendix). The results show that ESENSC achieves substantial computational gains over exact SHAP while closely matching its attribution patterns, often outperforming widely used SHAP approximation methods in both deviation and rank agreement. In contrast, PA--type methods exhibit instability and large deviations, consistent with their theoretical properties in XAI--TU games.

Furthermore, we provide theoretical foundations by establishing an axiomatization of ESENSC. We show that, for $n\neq 3$, this attribution rule is uniquely determined by efficiency, the null-player property, restricted differential marginality, pairwise interaction splitting and information parsimony. The substantive case is $n \ge 4$: for $n\le 2$, every game is 2-additive and ESENSC coincides with SHAP, whereas for $n = 3$ SHAP also satisfies all five axioms. This characterization shows that ESENSC can retain several SHAP--compatible axiomatic properties while achieving computational savings under an explicit linear information constraint. In this context, the solution concept most closely related to ESENSC is FESP proposed by \cite{Condevaux2023}; however, FESP is primarily heuristic and lacks a full axiomatic characterization.


ESENSC is particularly effective in high-dimensional settings where exact SHAP is computationally infeasible and sampling-based approximations may exhibit variability in their estimates, offering a scalable and theoretically grounded alternative. To the best of our knowledge, ESENSC is the first model-agnostic AFA that provides a unique axiomatic characterization for $n \ge 4$ under explicit computational constraints, thereby addressing a fundamental trade-off between interpretability, computational efficiency, and theoretical rigor.


In summary, this paper makes three main contributions. First, we formalize feature attribution problems as XAI--TU games and highlight structural properties specific to this setting. Second, we propose ESENSC, a computationally efficient AFA that achieves strong empirical agreement with SHAP while substantially reducing computational cost. Third, we establish an axiomatic characterization of ESENSC, clarifying how computational tractability can be achieved while preserving key desirable properties.
Together, these results position ESENSC as a principled and practical alternative to SHAP for feature attribution.

The remainder of the paper is organized as follows. Section \ref{subsec_junbi} introduces XAI--TU games. Section \ref{main_results} develops AFA methods. Section \ref{experiment} presents empirical results. Section \ref{axiomatization} provides the axiomatic characterization. Section \ref{conclusion} concludes.

\vspace{-5pt}
\section{Explainable AI TU Games (XAI--TU games)}\label{subsec_junbi}
 Let $t$ denote the number of observations and $n$ the number of features. Define the index sets $N=\{1,\ldots,n\}$ and $T=\{1,\ldots,t\}$. Let $X=(X_1,\ldots,X_j,\ldots,X_n)$ be the $t\times n$ feature matrix. For each feature $j\in N$, let $X_j=(x_{1,j},\ldots,x_{t,j})'$ denote the vector of its realizations across observations. For each observation $\tau\in T$, let $x_\tau=(x_{\tau,1},\ldots,x_{\tau,j},\ldots,x_{\tau,n})$ denote the feature vector of the $\tau$-th observation. Let $f$ be a trained prediction model, and let $Y=(y_1,\ldots,y_t)'$ denote the vector of predicted values given by $f$, i.e., $Y=f(X)$.

For any coalition $S\in 2^N$, define $x_{\tau,S}=\{x_{\tau,j}\mid j\in S\}$, that is, the subvector of the 
$\tau$-th observation consisting of the features in $S$. Similarly, define $X_S=\{X_j\mid j\in S\}$. Let $s=|S|$ denote the cardinality of $S$. Whenever no confusion arises, a singleton coalition $\{i\}$ is simply written as $i$.
In cooperative game theory, a transferable-utility (TU) game is represented by a pair $(N,v)$, where $N=\{1,\ldots,n\}$ is the set of players and $v$ is a real-valued function defined on $2^N$. A central problem in cooperative game theory is how to allocate the total worth $v(N)$ among players, given the worth $v(S)$ that each coalition $S\subseteq N$ can generate on its own.

We now construct, for each observation $\tau$, a TU game in which the set of features $N$ is the set of players. For the $\tau$-th observation, define a characteristic function $v_\tau:2^N\to\mathbb{R}$ by $v_\tau(S)=\mathbb{E}\left[f(x_{\tau,S},\,X_{N\setminus S})\right],$ where the expectation is taken with respect to a reference distribution of the features in $N\setminus S$. The TU game defined by this expression is referred to as a XAI--TU game.

Throughout this paper, we adopt the interventional formulation of the characteristic function. 
Specifically, for a coalition $S$, the feature values in $N\setminus S$ are drawn jointly from their empirical distribution, independently of $x_{\tau,S}$. Accordingly, $v_{\tau}(S)$ is defined as $\mathbb{E}[f(x_{\tau,S}, X_{N\setminus S})]
= \frac{1}{t}\sum_{\tau'\in T} f(x_{\tau,S}, x_{\tau',N\setminus S})$,
which corresponds to the interventional version of SHAP commonly used in the literature. We do not consider conditional expectations of the form $\mathbb{E}[f(X)\mid X_S=x_{\tau,S}]$.

$v_{\tau}(\emptyset)=\mathbb{E}[f(X)]$ represents the baseline prediction when all features are treated as unknown. Unlike standard cooperative game formulations where $v(\emptyset)=0$ is often assumed without loss of generality, in XAI--TU games $v_{\tau}(\emptyset)$ is generally nonzero and must be treated explicitly.

Another notable feature of the XAI--TU games is that the values $v_{\tau}(S)$ may differ in sign across coalitions. The total surplus to be allocated, namely the improvement of the prediction from its baseline level, $v_{\tau}(N)-v_{\tau}(\emptyset),$ may be either positive or negative. Similarly, marginal contributions such as $v_{\tau}(i)-v_{\tau}(\emptyset)$ may differ in sign across features (players). It is even possible for $\sum_{i\in N}\bigl(v_{\tau}(i)-v_{\tau}(\emptyset)\bigr)$ to have a different sign from $v_{\tau}(N)-v_{\tau}(\emptyset)$. Such coexistence of positive and negative contributions is not typical in traditional cooperative game theory but arises naturally in XAI--TU games. When standard solution concepts, typically proportional allocation rules, are applied directly, this may lead to unintuitive attribution patterns. These issues are examined in the next section.

For later analysis, we introduce Harsanyi dividends \cite{Harsanyi1963, shapley} and the class of
2-additive games. For a game $v$, let $\lambda^v_\emptyset = v(\emptyset)$ and,
for every nonempty $T \subseteq N$, $\lambda^v_T = v(T) - \sum_{R \subsetneq T} \lambda^v_R$;
equivalently, $\lambda^v_T = \sum_{R \subseteq T} (-1)^{|T|-|R|} v(R)$,
so that $v(S) = \sum_{R \subseteq S} \lambda^v_R$ for all $S \subseteq N$.
Since $v(\emptyset)$ is generally nonzero in XAI--TU games, it is absorbed
into $\lambda^v_\emptyset$; in particular, $\lambda^v_{\{i\}} = v(i) - v(\emptyset)$.
We write $\lambda^v_i$ for $\lambda^v_{\{i\}}$.
A game $v$ is 2-additive if $\lambda^v_T = 0$ for every $T \subseteq N$
with $|T| \ge 3$.

Finally, in many XAI applications the number of features $n$ is large. Since the exact computation of the Shapley value grows exponentially in $n$, computing exact SHAP values becomes computationally infeasible in most practical settings. This motivates the search for computationally efficient alternatives.

These characteristics distinguish XAI--TU games from standard cooperative games and motivate the development of new solution concepts tailored to this setting.

\vspace{-5pt}
\section{Solution Concepts for XAI--TU Games}\label{main_results}
\subsection{Additive Feature Attribution (AFA)}\label{def_of_afa}
In a XAI--TU game, a solution concept for a given observation $\tau$ allocates the difference $v_{\tau}(N)-v_{\tau}(\emptyset)$ across features according to their contributions. We refer to any such allocation rule as an additive feature attribution (AFA). In what follows, we suppress the index $\tau$ and write the associated TU game as $(N,v)$. For each feature $j\in N$, let $\psi_j\in\mathbb{R}$ denote its assigned contribution, and write $\psi=(\psi_1,\ldots,\psi_n)$.
We say that $\psi$ is an AFA if
\begin{equation*}\label{AFAdef}
\sum_{j\in N}\psi_j = v(N)-v(\emptyset).
\end{equation*}
That is, an AFA exactly distributes the difference between the full-information prediction and the baseline prediction. In the remainder of this section, we compare AFAs derived from cooperative game--theoretic solution concepts.

\subsection{SHAP}\label{shap}
SHAP in a XAI--TU game is defined as
\begin{equation*}
\psi_{j}^{\mathrm{SHAP}}
=
\sum_{S\subseteq N\setminus j}
\frac{s!(n-s-1)!}{n!}
\left(v(S\cup j)-v(S)\right).
\label{SHAP_saisho}
\end{equation*}
This coincides with the Shapley value of the associated XAI--TU game \cite{shapley,lundberg2017}.
SHAP has a strong theoretical foundation as the Shapley value, but its exact computation requires evaluating $v(S)$ for all $2^n$ coalitions, resulting in exponential complexity.

In cooperative game theory, equal surplus (ES) and proportional allocation (PA) rules are well known for their low computational cost. Building on these tractable solution concepts, we formulate corresponding AFAs and systematically analyze their properties in XAI--TU games. 

\subsection{AFAs Based on Equal Surplus and Related Solutions}\label{es-related}
The equal surplus solution (ES) is defined by
\begin{equation*}
\psi_{j}^{\mathrm{ES}}
=
v(j)-v(\emptyset)
+
\frac{
v(N)-v(\emptyset)
-
\sum_{i\in N}\left\{v(i)-v(\emptyset)\right\}
}{n}.
\label{ES}
\end{equation*}
The ES allocation proceeds in two steps. First, each player $j$ receives its marginal contribution relative to the empty coalition, $v(j)-v(\emptyset)$. Second, the residual surplus, defined as $v(N)-v(\emptyset)-\sum_{i\in N}(v(i)-v(\emptyset))$, is distributed equally among all players. The ES--type AFA $\psi_{j}^{\mathrm{ES}}$ was briefly noted in \cite{Condevaux2023,HIS2024}. Here, we further develop ES--type AFAs for XAI--TU games.

The first related AFA can be viewed as a ``reverse'' ES rule, known as the Egalitarian Nonseparable Contribution (ENSC), formalized and axiomatized in cooperative game theory (\cite{Dragan_Drissen_Funaki, Drissen_Funaki_1991, Moulin_1981, Moulin_1985}). While ES builds from the empty coalition by assigning marginal contributions $v(j)-v(\emptyset)$, ENSC starts from the grand coalition and assigns the marginal loss $v(N)-v(N\setminus j)$ to each player $j$. As in ES, the remaining surplus is then equally redistributed among all players.

Formally, the ENSC-based AFA is defined as
\begin{equation*}
\psi_{j}^{\mathrm{ENSC}}
=
v(N)-v(N\setminus j)
+
\frac{
v(N)-v(\emptyset)
-
\sum_{i\in N}
\left\{
v(N)-v(N\setminus i)
\right\}
}{n}.
\label{AFA-ENSC}
\end{equation*}

A further AFA can be obtained by taking the simple average of ES and ENSC. To facilitate later modifications, we explicitly refer to this unadjusted version as ESENSC-unadj:
\begin{equation*}
\psi_{j}^{\mathrm{ESENSC\text{-}unadj}}
=
\frac{
\psi_{j}^{\mathrm{ES}}
+
\psi_{j}^{\mathrm{ENSC}}}{2}.
\label{ESENSC}
\end{equation*}

A key advantage of ES--type solutions is their low computational cost. While SHAP requires evaluating $v(S)$ for all $2^n$ coalitions, ES and ENSC each require only $n+2$ evaluations, and ESENSC-unadj requires $2n+2$, which is still substantially smaller than $2^n$. As shown in the experiments, the differences in computation time among ES, ENSC, and ESENSC-unadj are negligible. Accordingly, we use ESENSC-unadj as a preliminary benchmark within the class of ES--type AFAs.

While ESENSC-unadj is computationally efficient, it has a key drawback: it typically violates the null-player property.

\begin{itemize}[leftmargin=*]
\item \textbf{Null--player property.}  
If $j\in N$ satisfies $v(S\cup j)=v(S)$ for all $S\subseteq N\setminus j$, then $\psi_j(v)=0$.\footnote{
We call player $j$ satisfying this condition a null player in $v$, and occasionally write $\psi_j$ as a function of its associated (XAI-)TU game $v$.}
\end{itemize}

This violation arises from the unconditional equal redistribution of residual surplus, which assigns non--zero payoffs to players with zero marginal contributions. To address this issue, we consider modified ESENSC rules that restrict redistribution to features with non--zero marginal effects.

While multiple modifications are possible, we focus on a parsimonious version that satisfies the null-player property:\footnote{
An alternative modification proceeds as follows: first assign $v(j)-v(\emptyset)$ and redistribute the residual surplus equally among players with $v(j)\neq v(\emptyset)$; then assign $v(N)-v(N\setminus j)$ and redistribute the residual among players with $v(N)\neq v(N\setminus j)$; finally, average the two outcomes. This rule is well defined if (i) $\exists j$ such that $v(j)\neq v(\emptyset)$ and (ii) $\exists j$ such that $v(N)\neq v(N\setminus j)$; otherwise it is undefined. In practice, these conditions rarely fail, so the rule behaves similarly to ESENSC. However, from a theoretical perspective, ESENSC is defined on a strictly larger class of games.
}
\begin{itemize}[leftmargin=*]
\item Step 1: Assign $v(j)-v(\emptyset)$ to each player $j$.
\item Step 2: Assign $v(N)-v(N\setminus j)$ to each player $j$.
\item Step 3: Take the average of these two quantities.
\item Step 4: Redistribute the remaining surplus equally among players whose values are nonzero in at least one of Steps 1 or 2.
\end{itemize}

Formally, the modified rule, called ESENSC, admits the following closed-form expression, which is equivalent to the above procedure. First, for each $j\in N$, define
\begin{equation*}\label{b_j_tuika}
b_j(v)=\frac{v(j)-v(\emptyset)+v(N)-v(N\setminus j)}{2}.
\end{equation*}
Second, let $\mathcal{A}(v)$ denote the set of players $k\in N$ for which $v(k)\neq v(\emptyset)$ or $v(N)\neq v(N\setminus k)$. We call a player in $\mathcal{A}(v)$ boundary--active and a player outside $\mathcal{A}(v)$ boundary--inactive. Moreover, let $\mathbb{V}=\{\, v:2^N\to\mathbb{R}\mid \mathcal{A}(v) \neq \emptyset \,\}.$ Hereafter, we restrict attention to XAI--TU games in $\mathbb{V}$. ESENSC is defined for games in $\mathbb{V}$.

Lastly, define the residual required for efficiency by
\begin{equation*}
R(v)=v(N)-v(\emptyset)-\sum_{i\in \mathcal A(v)} b_i(v).
\label{eq:euclidean-residual}
\end{equation*}
Since $b_i(v)=0$ for every $i\notin \mathcal A(v)$, this is equivalent to subtracting $\sum_{i\in N}b_i(v)$ from the total surplus.

With these definitions, ESENSC assigns $\psi^{\mathrm{ESENSC}}_j=0$ to every $j\notin\mathcal{A}(v)$, and for $j\in\mathcal{A}(v)$,
\begin{equation}\label{ESENSC-r2}
\psi^{\mathrm{ESENSC}}_{j}
= b_j(v)+\frac{R(v)}{|\mathcal{A}(v)|}.
\end{equation}
This modified AFA satisfies the null-player property. In fact, ESENSC satisfies a stronger condition, possibly assigning zero to some players who are not null players. This condition is not imposed for normative reasons. Rather, under computational constraints, it allows features that may behave as null players (though they are not necessarily true null players) to be assigned zero attribution. In the subsequent analysis, we adopt ESENSC as the representative ES--type AFA.\footnote{ESENSC is not continuous in general because the boundary--active set $\mathcal{A}(v)$, and hence the set over which the residual is redistributed, may change discontinuously as $v$ varies. On regions where $\mathcal{A}(v)$ remains unchanged, ESENSC is continuous.}

The equal residual redistribution in ESENSC also admits a simple optimization interpretation. Conditional on assigning zero attribution to players outside $\mathcal{A}(v)$, ESENSC is the unique efficient allocation that minimizes the squared Euclidean adjustment from the preliminary boundary scores $b_i(v)$. Thus, equal residual redistribution is the minimum-Euclidean correction required to restore efficiency. Appendix A provides a formal statement and proof.

It is also useful to relate ESENSC directly to SHAP through Harsanyi dividends introduced in Section \ref{subsec_junbi}, which decompose a game into coalition--specific value components. Appendix B shows that ESENSC coincides with exact SHAP for every 2--additive game. More generally, when all players are boundary-active, any ESENSC--SHAP difference is generated by Harsanyi-dividend components associated with coalitions of size three or larger. Appendix B provides the exact decomposition, a corresponding bound, and a characterization of the boundary-inactive case.

\subsection{AFAs Based on Proportional Solutions and Related Solutions}\label{prop-related}
We next consider AFAs based on proportional allocation (PA).\footnote{The solutions discussed in this subsection are defined only for games in which the denominator is not zero.} The PA rule is defined as:
\begin{equation}
\psi^{\mathrm{PA}}_{j}
= \frac{v(j)-v(\emptyset)}
{\sum_{i\in N}\left\{v(i)-v(\emptyset)\right\}}
\cdot
\left\{v(N)-v(\emptyset)\right\}.
\label{PA}
\end{equation}
Furthermore, as in Subsection \ref{es-related}, one may also consider a ``reverse'' proportional rule defined by the following (\cite{straffin_1981, kongo_2026, li_2020, Brink_2023}):
\begin{equation*}
\psi^{\mathrm{ROP}}_{j}
=
\frac{v(N)-v(N\setminus j)}
{\sum_{i\in N}\left\{v(N)-v(N\setminus i)\right\}}\cdot\left\{v(N)-v(\emptyset)\right\}.
\label{ROP}
\end{equation*}
and the simple average of PA and ROP:
\begin{equation*}
\psi^{\mathrm{PAROP}}_{j}
=
\frac{
\psi^{\mathrm{PA}}_{j}
+
\psi^{\mathrm{ROP}}_{j}}{2}.
\label{PAROP}
\end{equation*}

Those PA--type rules satisfy the null-player property but face a fundamental difficulty in XAI--TU games. To illustrate, consider PA (same arguments hold for ROP and PAROP). As noted in Section \ref{subsec_junbi}, $\sum_{i\in N}(v(i)-v(\emptyset))$ and $v(N)-v(\emptyset)$ may have opposite signs. In such cases, $v(i)-v(\emptyset) < v(j)-v(\emptyset)$ with $v(i)-v(\emptyset) \neq 0$ and $v(j)-v(\emptyset) \neq 0$ may imply $\psi^{\mathrm{PA}}_{i}>\psi^{\mathrm{PA}}_{j}$, reversing the ordering of marginal contributions.\footnote{We exclude potential null players (i.e., $i$ with $v(i)-v(\emptyset)=0$) from the order comparison. Under PA, such players receive zero attribution; moreover, efficiency, the null-player property, and order preservation extended to potential null players are incompatible in general.} We refer to this as the order-reversal problem, illustrated in the following example.

\begin{ex}\label{ex0}
Consider the two-player game given by  
$v(\emptyset)=0$, $v(1)=-3$, $v(2)=2$, and $v(\{1,2\})=1$.  
Then
$\psi^{\mathrm{PA}}_{1}=3$ and 
$\psi^{\mathrm{PA}}_{2}=-2$.
Feature 1 decreases the prediction, whereas feature 2 increases it. Yet PA assigns a large positive contribution to feature 1 and a large negative one to feature 2.
\end{ex}

If an AFA does not induce the order-reversal problem, we say that it satisfies the order preservation property. As a PA--type AFA which satisfies this property, we define the following PARPA. Appendix C presents the derivation of PARPA and explains why it avoids this problem:
\begin{equation*}
\psi^{\mathrm{PARPA}}_{j} =
\begin{cases}
\psi^{\mathrm{PA}}_{j}
&
\text{if }
\big[v(N)-v(\emptyset)\big]
\cdot
\big[\sum_{i\in N}\{v(i)-v(\emptyset)\}\big]
>0,
\\[6pt]
\psi^{\mathrm{RPA}}_{j}
&
\text{otherwise,}
\end{cases}
\end{equation*}
where 
\begin{eqnarray*}
\psi^{\mathrm{RPA}}_{j}
&=&
\frac{
v(N)-v(\emptyset)-\sum_{i\neq j}\left\{v(i)-v(\emptyset)\right\}
}{
\sum_{i\in N}
\left\{
v(N)-v(\emptyset)-\sum_{k\neq i}\left\{v(k)-v(\emptyset)\right\}
\right\}
}
\cdot
\left\{v(N)-v(\emptyset)\right\}
\end{eqnarray*}

\vspace{-5pt}
\section{Experiments}\label{experiment}
We evaluate the AFA methods introduced in Section \ref{main_results} by comparing deviation from exact SHAP and computational cost. Experiments in this section are based on the California Housing dataset, whereas supplementary results using the HELOC dataset are presented in Appendix D.
\paragraph{Dataset.} We use the California Housing dataset,\footnote{Available in \texttt{scikit-learn} via the \texttt{fetch\_california\_housing} function (1990 U.S.\ Census data).} a widely used benchmark for regression tasks, with median house value as the target variable. To study scalability, we augment the original eight features with additional synthetic features drawn from a standard normal distribution and standardize all inputs. The number of features is increased up to 512, allowing us to examine both approximation accuracy and computational scaling as the dimensionality grows.
\paragraph{Models.} 
We consider two widely used nonlinear regression models: XGBoost and a feedforward neural network. For both models, hyperparameters are tuned via cross-validation over standard ranges, following common practice in tabular prediction tasks. For more details, see Appendix E. 
\paragraph{Attribution Methods.} 
We compare three categories: (i) exact SHAP as the benchmark, (ii) our AFA methods from Section \ref{main_results} (including ESENSC and PARPA as representative ES--type and PA--type rules, along with related variants), and (iii) widely used SHAP approximations, namely Kernel SHAP (\cite{lundberg2017}), Permutation SHAP (\cite{Strumbelj_2010, Strumbelj_2014}),\footnote{For Permutation SHAP, the number of model evaluations is fixed at $(2n+1)\times 10$ to ensure comparability across feature dimensions.} and TreeSHAP (\cite{TreeExplainer}). For TreeSHAP, we employ two different settings: (i) TreeSHAP (auto), which uses the default conditional distribution without background data, and (ii) TreeSHAP (interv.), which utilizes background data for interventional expectations.

\subsection{Deviations from (Exact) SHAP}\label{dev_from_shap}
We first evaluate AFAs in terms of the deviation from exact SHAP using multiple complementary metrics that capture both magnitude differences and ranking consistency.

The first measure is Normalized Attribution Distance (NAD) from exact SHAP. For AFA $\psi^{AFA}$, 
\begin{equation}\label{NAD}
\mathrm{NAD}(\psi^{AFA})
=\frac{\sum_{\tau=1}^t \sum_{j=1}^n\left|\psi^{AFA}_{\tau j} - \psi^{SHAP}_{\tau j}\right|}{\sum_{\tau=1}^t
\left|v_{\tau}(N)-v_{\tau}(\emptyset)\right|}.
\end{equation}
NAD captures the overall difference in attribution magnitudes, normalized by the scale of predictions.

The next measure is Mean Rank Correlation (MRC) with exact SHAP, defined as
\begin{equation}\label{MRC}
\mathrm{MRC}(\psi^{AFA})=
\frac{1}{t}\sum_{\tau=1}^t \rho_\tau^{(\psi^{AFA},\psi^{SHAP})},
\end{equation}
where $\rho_\tau^{(\psi^{AFA},\psi^{SHAP})}$ denotes the Pearson correlation of the within-observation ranks of $\{\psi^{AFA}_{\tau j}\}_{j=1}^n$ and $\{\psi^{SHAP}_{\tau j}\}_{j=1}^n$.
Together with NAD, this measure evaluates agreement with exact SHAP in both attribution magnitude and relative feature importance.

\begin{figure}[tb]
    \centering
\includegraphics[width=1\linewidth]{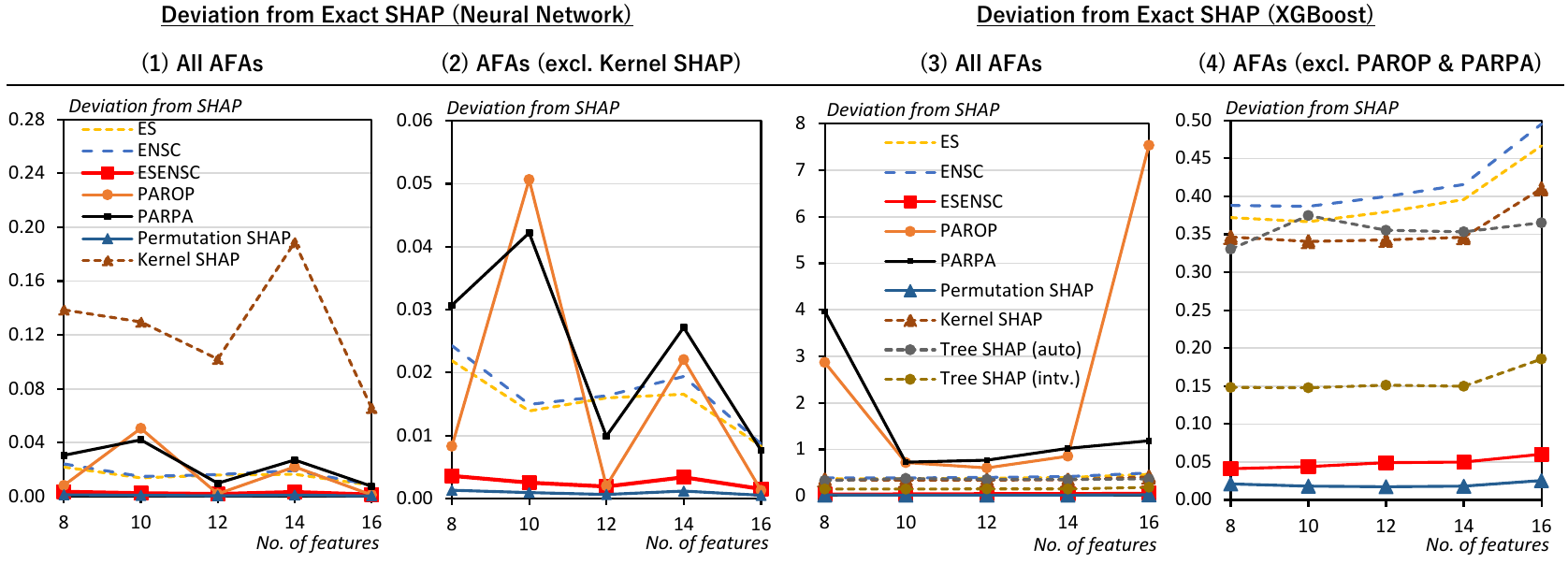}
 \caption{NAD values (deviation from exact SHAP) on the California Housing dataset. The left two panels correspond to Neural Network and the right two panels correspond to XGBoost.}
    \label{chr_dev_from_SHAP_01}
\end{figure}

Figure \ref{chr_dev_from_SHAP_01} reports the NAD of AFAs where smaller values indicate closer agreement. 
The panels (1) and (2) correspond to the neural network (NN). In panel (1), which includes all AFAs, Kernel SHAP exhibits relatively large deviations from exact SHAP. The right panel, excluding Kernel SHAP, reveals clearer patterns: the PA--type AFAs (PAROP and PARPA) exhibit both large and volatile deviations, whereas the ES--type AFA ESENSC consistently achieves small deviations and strong approximation performance. The panels (3) and (4) show the results for XGBoost. The overall patterns are similar: PA--type AFAs again display large and unstable deviations, while ESENSC maintains consistently small deviations. ESENSC also performs favorably relative to TreeSHAP in this setting.

Why do PA--type AFAs exhibit such volatile behavior? Although PARPA was designed to address the order-reversal problem, it shows deviation patterns similar to PAROP, suggesting that additional sources of instability remain. A key mechanism arises when positive and negative contributions coexist. As an example, consider a two-player TU game with $v(\emptyset)=0$, $v(1)=10$, $v(2)=-9$, and $v(N)=15$. Although no order reversal occurs, the proportional scaling factor $v(N)/(v(1)+v(2))=15$ becomes large because the denominator is close to zero, leading to extreme allocations ($\psi_1^{PA}=150$, $\psi_2^{PA}=-135$). This illustrates that resolving order reversal alone does not prevent large distortions in PA--type methods.

\begin{figure}[tb]
    \centering
\includegraphics[width=1\linewidth]{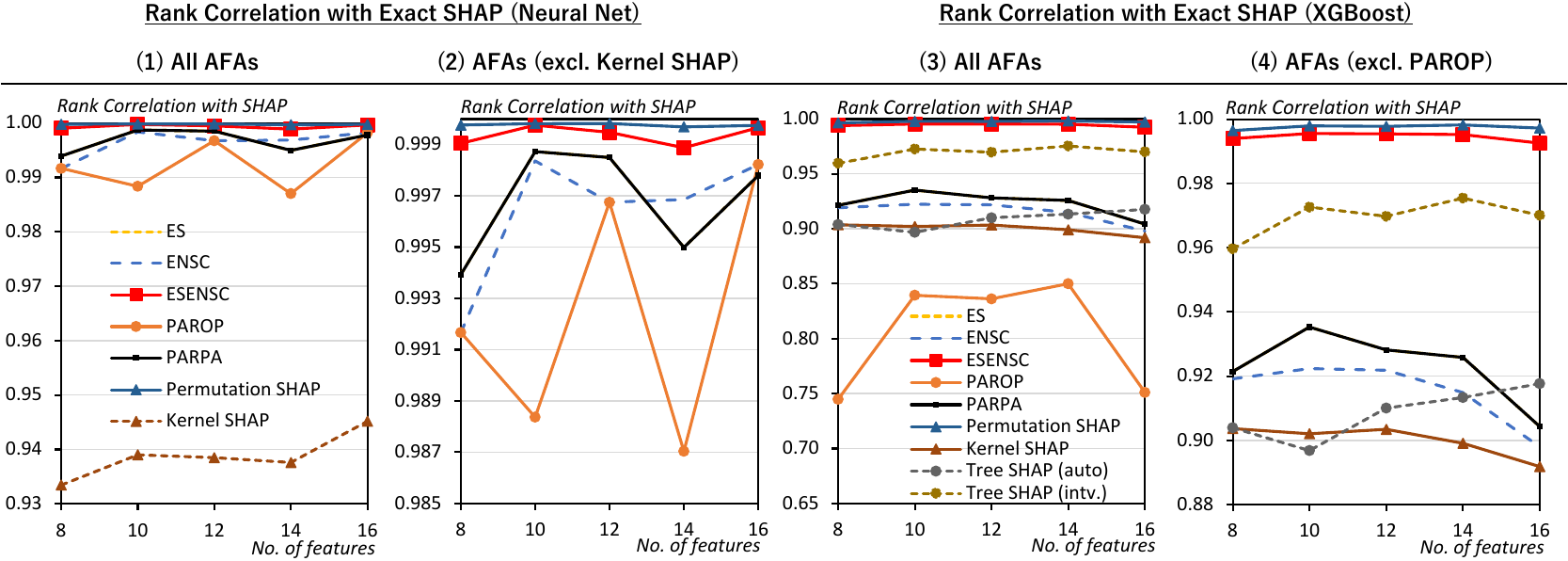}
 \caption{MRC values (rank correlation with exact SHAP) on the California Housing dataset. The left two panels correspond to Neural Network and the right two panels correspond to XGBoost.}
    \label{chr_dev_from_SHAP_02}
\end{figure}

Figure \ref{chr_dev_from_SHAP_02} reports the MRC of AFAs where higher values indicate closer agreement with exact SHAP. As in the case of NAD, panels (1) and (2) show the results for NN, while panels (3) and (4) for XGBoost. The insights are broadly consistent with those from the NAD results: PA--type AFAs exhibit large and volatile deviations from exact SHAP. In contrast, ESENSC closely approximates exact SHAP and achieves better performance than many SHAP approximation methods. 

\subsection{Computation Time}\label{comp_cost}
We compare the computational cost of the attribution methods as a function of the number of features. Figure~\ref{chr_CompTime} shows the results for a neural network model, with panel (1) covering feature dimensions up to 16 where exact SHAP is computable.
As expected, its computation time grows exponentially with the number of features, while all other AFA methods require substantially less time. Panel (2) extends the comparison to higher dimensions (up to 512 features), where computation of exact SHAP becomes infeasible. The results show that the proposed AFAs scale approximately linearly and remain computationally efficient even in high-dimensional settings. Their computation time is substantially lower than that of sampling-based methods such as Permutation SHAP and Kernel SHAP.
The panels (3) and (4) report the results for XGBoost. The overall patterns are consistent with those for the neural network. 

Although the runtime of sampling-based SHAP methods can be reduced by limiting the number of model evaluations, this typically worsens approximation accuracy, implying a trade-off between computational cost and precision. In contrast, the proposed AFAs are parameter-free once coalition values are specified, allowing direct and efficient computation.

Overall, under the evaluated settings, ESENSC offers a favorable trade-off between agreement with exact SHAP and computational cost, and compares favorably with the sampling-based SHAP approximations considered here.

\begin{figure}[tb]
    \centering
\includegraphics[width=1\linewidth]{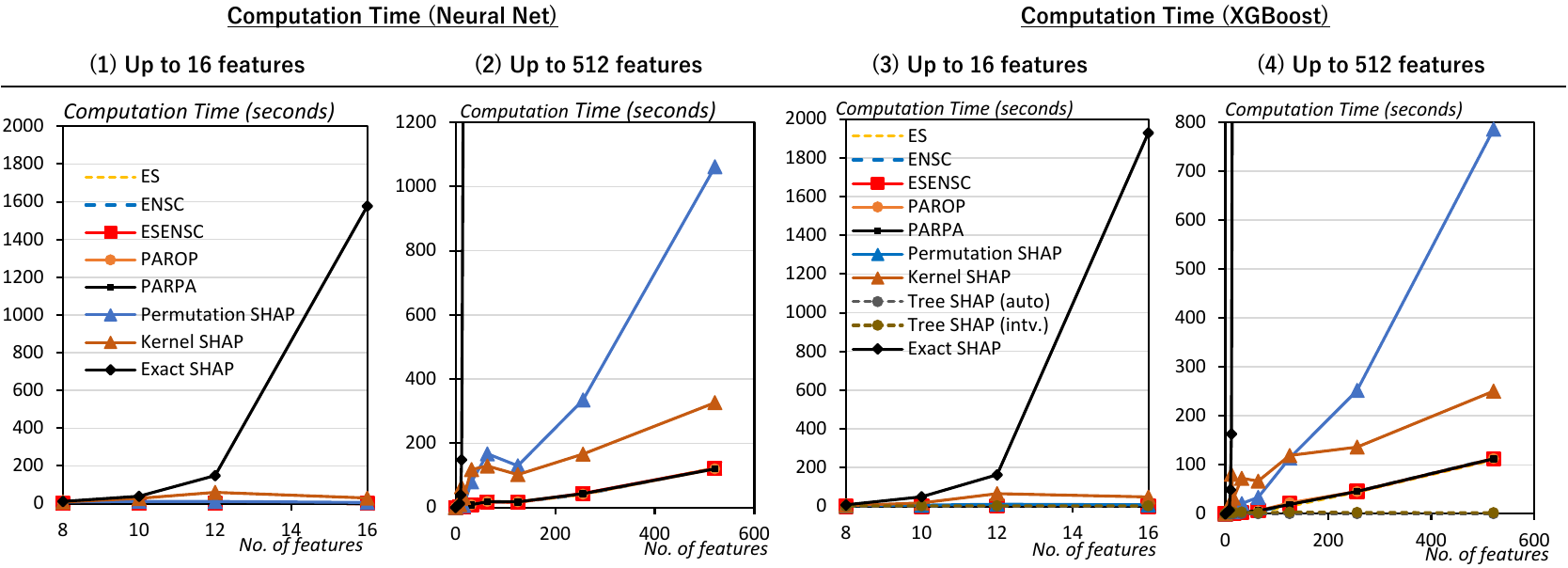}
 \caption{Computation time on the California Housing dataset. The left two panels correspond to Neural Network and the right two panels correspond to XGBoost.}
    \label{chr_CompTime}
\end{figure}

\vspace{-5pt}
\section{Axiomatization of ESENSC}\label{axiomatization}
This section presents the theoretical foundations of ESENSC and its axiomatic characterization, showing that, for $n \neq 3$, it is uniquely determined by axioms tailored to XAI--TU games, analogous to classical Shapley characterizations.

ESENSC satisfies all five axioms in Theorem \ref{maintheoremforaxiom} for every $n$. The role of $n$ can be summarized as follows. (i) For $n \le 2$, every game is 2-additive, so ESENSC coincides with SHAP on $\mathbb{V}$ (Corollary \ref{cor:two-additive} in Appendix B); the five axioms pin down this common rule, and the characterization does not distinguish ESENSC from SHAP. (ii) For $n = 3$, we have $2n + 2 = 2^n$, so information parsimony imposes no restriction; SHAP then also satisfies all five axioms, while ESENSC and SHAP differ whenever $\mathcal{A}(v) \neq N$ and $\lambda^v_N \neq 0$ (Appendix B). Hence uniqueness fails. (iii) For $n \ge 4$, information parsimony becomes binding ($2n + 2 < 2^n$), and ESENSC is the unique rule satisfying the five axioms. Theorem \ref{maintheoremforaxiom} is stated for $n \neq 3$, but its substantive content lies in case (iii). For $n \ge 4$, the restriction to $\mathbb{V}$ is justified in Remark \ref{rem_nnoteq3} and the related discussion in Appendix G.


To state the main result, we introduce the following definition. 

\begin{defi} 
$S \subseteq N$ is relevant to $\psi$ if 
there exist $v,w \in \mathbb{V}$ such that $v(T)=w(T)$ for any $T \subseteq N$ with $T \neq S$, and $\psi(v) \neq \psi(w)$.  
\end{defi}

This definition captures the idea that a coalition $S$ is relevant to a solution $\psi$ if changing $v(S)$ alone, while keeping the worths of all other coalitions $T\subseteq N$ with $T\neq S$ unchanged, can alter the payoff vector assigned by $\psi$. In other words, $S$ is relevant to $\psi$ if its worth $v(S)$ can affect the payoffs determined by $\psi$. 
For example, every coalition $S\subseteq N$ is relevant to SHAP. In contrast, only $\emptyset$ and $N$ are relevant to the equal division solution, which assigns $\frac{v(N)-v(\emptyset)}{n}$ to each player. Finally, no coalition is relevant to the zero solution, defined by $\psi_i(v)=0$ for every $v\in\mathbb V$ and $i\in N$.

The main result of this section is the following: 
\begin{theo}\label{maintheoremforaxiom}
Let $n \neq 3$. An attribution rule $\psi$ satisfies the following five axioms---efficiency, the null player property, restricted differential marginality, pairwise interaction splitting, and information parsimony---if and only if $\psi=\psi^{\mathrm{ESENSC}}$.
\setlength{\leftmargini}{1.2em}
\begin{itemize}
\item \textbf{Efficiency.}  
For any $v \in \mathbb{V}$,  $\sum_{i \in N}\psi_i(v)=v(N)-v(\emptyset).$

\item \textbf{Null player property.}  
If $i \in N$ satisfies $v(S \cup i)=v(S)$ for any $S \subseteq N \setminus i$, then $\psi_i(v)=0$.

\item \textbf{Restricted differential marginality.}
Let $v,w \in \mathbb{V}$ and $i,j \in N$ satisfy (i) $v(S \cup i)-v(S \cup j)=w(S \cup i)-w(S \cup j)$ for any $S \subseteq N \setminus \{i,j\}$, (ii) $i,j \in \mathcal{A}(v)$ and $i,j \in \mathcal{A}(w)$, and 
(iii) $\mathcal{A}(v)=\mathcal{A}(w)$. Then $\psi_i(v)-\psi_j(v)=\psi_i(w)-\psi_j(w).$
\item \textbf{Pairwise interaction splitting.} If $v \in \mathbb{V}$ is 2--additive, then $\psi_i(v)=\lambda^v_i+\sum_{j \neq i}\frac{\lambda^v_{\{i,j\}}}{2}$ for any $i \in N$, where $\lambda^v$ denotes the Harsanyi dividends of $v$ defined in Section \ref{subsec_junbi}.  
\item \textbf{Information parsimony.}  
$|\{S \subseteq N | S \text{ is relevant to } \psi \}| \le 2n+2$. 
\end{itemize}
\end{theo}
As for the proof of Theorem \ref{maintheoremforaxiom}, see Appendix G.  Appendix H shows independence of the axioms in Theorem \ref{maintheoremforaxiom}.

Here we compare the axioms in Theorem \ref{maintheoremforaxiom} for ESENSC with those used to axiomatize SHAP in \cite{Casajus11},
which are restated in Remark \ref{sankocasajus} in Appendix I. 
The first four axioms are all satisfied by SHAP. For $n \ge 4$, the distinction comes from the information parsimony axiom. Information parsimony limits the amount of coalition-value information on which the solution may depend. The particular bound $2n+2$ can be interpreted as a simple linear information budget. 
When $n=3$, we have $2n+2=8=2^n$, so the axiom imposes no information restriction: all coalition values may be relevant. Taking the three-player case as the benchmark, the upper bound on the number of relevant coalition values increases by two for each additional player. 
This linear continuation gives $8+2(n-3)=2n+2$. 
Thus, this axiom does not prescribe which coalition values should matter, nor does it require the actual number of relevant coalition values to increase with $n$; it only imposes a linear upper bound calibrated at the unrestricted three-player benchmark. 
This calibration is also sharp: Lemma \ref{lemma_01} below shows that, for $n \ge 3$, the first four axioms require $2n+2$ coalition values to be relevant. 

The remaining four axioms are all compatible with SHAP. 
Efficiency and the null player property are minimal requirements for any attribution rule. 
For differential marginality, SHAP satisfies full differential marginality, requiring that identical marginal differences across coalitions imply identical attribution differences. In contrast, ESENSC satisfies only a restricted form of this axiom. More specifically, condition (ii) requires the two players being compared to be boundary-active in both games, while condition (iii) keeps the set of boundary-active players unchanged across the two games. Thus, restricted differential marginality compares payoff differences while holding fixed the set of players among whom the residual is redistributed. 

Pairwise interaction splitting applies only to 2-additive games, in which all interactions are individual or pairwise. Within this restricted class, the axiom assigns each singleton dividend to the corresponding player and equally divides each pairwise dividend between the two players involved. 
The class of 2-additive games is the simplest extension of additive games in terms of interaction order that allows genuine interactions. 
By restricting the requirement to this class, the axiom imposes no direct restriction on the treatment of higher-order interactions and is therefore a relatively weak requirement. 

\begin{lemm}\label{lemma_01}
Let $\psi$ satisfy efficiency, the null player property, restricted differential marginality and pairwise interaction splitting. 
Then, every coalition in 
$$
\mathcal{C}=\{S \subseteq N | |S|=0,1,n-1,n\}
$$
is relevant to $\psi$ on $\mathbb{V}$. 
\end{lemm}





Combining Lemma \ref{lemma_01} with information parsimony yields the following boundary-information result. 

\begin{coro}\label{corollary_01}
If $\psi$ satisfies the five axioms in Theorem \ref{maintheoremforaxiom}, then for any $v,w \in \mathbb{V}$, if $v(S)=w(S)$ for $|S|=0,1,n-1$ and $n$, then $\psi(v)=\psi(w)$. 
\end{coro}

We also use the following consequence of the five axioms in the proof of Theorem \ref{maintheoremforaxiom}. 

\begin{coro}\label{corollary_02}
If $\psi$ satisfies the five axioms in Theorem \ref{maintheoremforaxiom}, then for any $v \in \mathbb{V}$ with $v(N)-v(\emptyset)= \sum_{j \in N}\frac{1}{2}\Bigl(v(j)-v(\emptyset)+v(N)-v(N \setminus j)\Bigr)$, 
it holds that $\psi_i(v)=\frac{1}{2}\Bigl(v(i)-v(\emptyset)+v(N)-v(N \setminus i)\Bigr)$ for any $i \in N$.
\end{coro}

\vspace{-5pt}
\section{Conclusion}\label{conclusion}
This paper develops ESENSC, a computationally efficient and theoretically grounded alternative to SHAP for additive feature attribution. We demonstrate that ESENSC closely matches the attribution patterns of exact SHAP while achieving substantial computational gains, scaling linearly with the number of features.
On the theoretical side, we establish an axiomatic characterization of ESENSC for $n\neq 3$, which is substantive for $n\ge 4$ (for $n\le 2$, ESENSC coincides with SHAP), showing that it is uniquely determined by a set of properties tailored to XAI--TU games. This result shows how an explicit linear information constraint can yield substantial computational savings while retaining several axiomatic properties shared with SHAP. Empirically, under the evaluated settings, ESENSC performs competitively with, and in some settings better than, widely used SHAP approximation methods in terms of agreement with exact SHAP, while requiring substantially less computation and offering a theoretical foundation through its axiomatic characterization.


Overall, these results position ESENSC as a principled and practical alternative to SHAP, enabling reliable feature attribution for modern, high-dimensional models where exact computation is infeasible. An important direction for future research is to extend ESENSC beyond tabular data. Another important direction is to systematically understand the conditions under which ESENSC and SHAP yield different attribution patterns. This is particularly relevant because ESENSC focuses on marginal contributions at the boundary coalitions and may therefore underrepresent certain interaction effects that arise in intermediate coalitions.

\newpage
\begin{ack}
We thank Yukihiko Funaki, Kiyoshi Izumi, Sungmin Kan, Shigeo Muto, Yuya Takada, Koji Takahashi, Makoto Yokoo and participants at European Meeting on Game Theory 2026 and Asian Game Theory Conference 2026 for their insightful comments and suggestions. The views expressed in this paper are those of the authors and do not necessarily represent the views of the IMF, its Executive Board, or IMF management.

This work was supported by JSPS KAKENHI Grant Numbers 25K05176 and 25K05018, and Nomura Foundation.

The authors declare no competing interests.
\end{ack}

\bibliographystyle{abbrv}
\bibliography{SHAP} 

\section*{Appendix}
\subsection*{A. Optimization interpretation of equal residual redistribution}
\begingroup

This appendix formalizes the optimization interpretation discussed in Section \ref{es-related}.

\begin{prop}[Minimum-Euclidean efficiency correction]
\label{prop:euclidean-adjustment}
Suppose $v \in \mathbb{V}$. Among all attribution vectors $\psi\in\mathbb R^n$ that satisfy
\begin{equation}
\sum_{i\in N}\psi_i=v(N)-v(\emptyset),
\qquad \psi_i=0 \quad \text{for all } i\notin \mathcal A(v),
\label{eq:euclidean-feasible}
\end{equation}
ESENSC defined in (\ref{ESENSC-r2}) is the unique solution to
\begin{equation}
\min_{\psi\in\mathbb R^n}
\sum_{i\in \mathcal A(v)}\bigl(\psi_i-b_i(v)\bigr)^2
\quad\text{subject to \eqref{eq:euclidean-feasible}}.
\label{eq:euclidean-program}
\end{equation}
\end{prop}

\textbf{Proof.}
Because $\psi_i=0$ for $i\notin \mathcal A(v)$, the efficiency constraint is $\sum_{i\in \mathcal A(v)}\psi_i=v(N)-v(\emptyset)$. The Lagrangian for the boundary--active players is
\begin{equation*}
\mathcal L(\psi,\mu) =\sum_{i\in \mathcal A(v)}\bigl(\psi_i-b_i(v)\bigr)^2+\mu\left(\sum_{i\in \mathcal A(v)}\psi_i-v(N)+v(\emptyset)\right).
\end{equation*}
The first-order condition for each $i\in \mathcal A(v)$ is
\begin{equation*}
2\bigl(\psi_i-b_i(v)\bigr)+\mu=0,
\end{equation*}
so every boundary--active player receives the same adjustment, say $c$. Imposing efficiency gives
\begin{equation*}
|\mathcal A(v)|c=v(N)-v(\emptyset)-\sum_{i\in \mathcal A(v)}b_i(v)=R(v),
\end{equation*}
which yields $c=R(v)/|\mathcal A(v)|$ and hence (\ref{ESENSC-r2}). The objective in \eqref{eq:euclidean-program} is strictly convex in the boundary--active players and the constraint set is affine, so this minimizer is unique. \qed

\endgroup

\subsection*{B. Harsanyi-dividend characterization of the ESENSC--SHAP difference}
\begingroup
This appendix provides a complementary characterization of the difference between ESENSC and exact SHAP in terms of Harsanyi dividends $\lambda^v_T$ defined in Section \ref{subsec_junbi}. 
Because $v(\emptyset)$ is absorbed into $\lambda^v_{\emptyset}$, no normalization of $v(\emptyset)$ is required. In particular, the analysis applies directly to XAI--TU games with nonzero baseline values. The Shapley attribution admits the standard dividend representation
\begin{equation}
\psi_i^{\mathrm{SHAP}}(v)
= \sum_{T\subseteq N:\, i\in T} \frac{\lambda^v_T}{|T|}.
\label{eq:shap-harsanyi}
\end{equation}
We refer to $\lambda_T^v$ as the coalition-specific Harsanyi-dividend component associated with $T$ and call $|T|$ its order. It captures the component of the game associated with coalition $T$ after accounting for all proper subcoalitions. Components with $|T|\ge 3$ are referred to as higher-order Harsanyi-dividend components. Recall from Section \ref{es-related} that players in $\mathcal A(v)$ are called boundary-active.

The results below are organized by the boundary-activity structure of $v$. Propositions \ref{prop:harsanyi-difference} and \ref{prop:higher-order-bound} treat the case $\mathcal A(v) = N$, in which the residual is shared equally among all $n$ players. Corollaries \ref{cor:two-additive} and \ref{monotone 2ad} hold for every $v \in \mathbb V$ without this assumption. Proposition \ref{prop:false-null} characterizes the boundary-inactive case $i \notin \mathcal A(v)$.

\begin{prop}[Exact ESENSC--SHAP difference when all players are boundary-active]
\label{prop:harsanyi-difference}
Suppose that $v\in\mathbb V$ and $\mathcal A(v)=N$. Then, for every $i\in N$,
\begin{align}
\psi_i^{\mathrm{ESENSC}}(v)-\psi_i^{\mathrm{SHAP}}(v)
={}&
\sum_{\substack{T\subseteq N:\, i\in T\\ |T|\ge 2}}
\left(\frac{1}{2}-\frac{1}{|T|}\right)\lambda_T^v
\nonumber\\
&+\frac{1}{n}
\sum_{\substack{T\subseteq N\\ |T|\ge 2}}
\left(1-\frac{|T|}{2}\right)\lambda_T^v.
\label{eq:esensc-shap-dividend}
\end{align}
Equivalently, for $2\le \tau\le n-1$, define
\begin{equation*}
I(i,\tau) = \frac{1}{\binom{n-1}{\tau-1}} \sum_{\substack{T\subseteq N:\, i\in T\\ |T|=\tau}}
\lambda_T^v, \qquad J(i,\tau) = \frac{1}{\binom{n-1}{\tau}}
\sum_{\substack{T\subseteq N:\, i\notin T\\ |T|=\tau}}\lambda_T^v,
\label{eq:IJ-def}
\end{equation*}
where $I(i,\tau)$ (resp. \ $J(i,\tau)$) is the average of the Harsanyi dividends of order $\tau$ that include (resp. \ do not include) $i$. Then
\begin{equation}
\psi_i^{\mathrm{ESENSC}}(v)-\psi_i^{\mathrm{SHAP}}(v)
=
\sum_{\tau=2}^{n-1}
\binom{n-1}{\tau-1}
\left(\frac{1}{\tau}-\frac{1}{n}\right)
\left(1-\frac{\tau}{2}\right)
\bigl[J(i,\tau)-I(i,\tau)\bigr].
\label{eq:esensc-shap-by-order}
\end{equation}
In particular, the grand-coalition dividend cancels from the difference.
\end{prop}

\textbf{Proof.}
From the dividend representation in Section \ref{subsec_junbi},
\begin{equation}
b_i(v) = \lambda^v_{i} + \frac{1}{2}\sum_{\substack{T\subseteq N:\, i\in T\\ |T|\ge 2}}\lambda^v_T.
\label{eq:bi-harsanyi}
\end{equation}\label{prfprop2_03}
Moreover,
\begin{equation}\label{prfprop2_04}
\sum_{j\in N}b_j(v)=\sum_{j\in N}\lambda^v_{j}+\frac{1}{2}\sum_{\substack{T\subseteq N\\ |T|\ge 2}}|T|\lambda^v_T,
\end{equation}
whereas
\begin{equation}\label{prfprop2_05}
v(N)-v(\emptyset)=\sum_{\emptyset\neq T\subseteq N}\lambda_T^v.
\end{equation}
Note that (\ref{eq:bi-harsanyi}) to (\ref{prfprop2_05}) hold for every $v\in\mathbb V$ and do not use the assumption $\mathcal A(v) = N$. Because $\mathcal A(v) = N$, the residual is redistributed over all $n$ players. Therefore,
\begin{equation*}
\psi_i^{\mathrm{ESENSC}}(v)
= \lambda^v_{i} + \frac{1}{2} \sum_{\substack{T\subseteq N:\, i\in T\\ |T|\ge 2}} \lambda_T^v + \frac{1}{n}
\sum_{\substack{T\subseteq N\\ |T|\ge 2}}\left(1-\frac{|T|}{2}\right)\lambda_T^v.
\label{eq:esensc-harsanyi}
\end{equation*}
Subtracting \eqref{eq:shap-harsanyi} gives \eqref{eq:esensc-shap-dividend}. Regrouping the terms in \eqref{eq:esensc-shap-dividend} by coalition size and using
$\binom{n-1}{\tau-1}/\binom{n-1}{\tau}=\tau/(n-\tau)$
yields \eqref{eq:esensc-shap-by-order}. For $T=N$, the coefficient is
$(1/2-1/n)+(1/n)(1-n/2)=0$, so the grand-coalition dividend cancels. \qed

Equation \eqref{eq:esensc-shap-by-order} makes the role of Harsanyi-dividend order explicit. The factor $1-\tau/2$ is zero at $\tau=2$, so second-order Harsanyi-dividend components do not generate an ESENSC--SHAP difference. More generally, if $I(i,\tau)=J(i,\tau)$ at every relevant higher order, then the two attributions coincide for feature $i$ even when higher-order Harsanyi-dividend components are present.

\begin{prop}[Bound on deviations induced by higher-order dividend components]\label{prop:higher-order-bound}
Suppose again that $\mathcal A(v)=N$ and define
$\Delta_i=\psi_i^{\mathrm{ESENSC}}(v)-\psi_i^{\mathrm{SHAP}}(v)$. Let
\begin{equation*}
\mathcal H_i^+ = \{T\subseteq N: i\in T,\ 3\le |T|\le n-1\},
\qquad \mathcal H_i^- = \{T\subseteq N: i\notin T,\ 3\le |T|\le n-1\}.
\end{equation*}
Then
\begin{equation}
\Delta_i
=
\sum_{T\in\mathcal H_i^+}
\frac{(|T|-2)(n-|T|)}{2|T|n}\lambda_T^v - \sum_{T\in\mathcal H_i^-} \frac{|T|-2}{2n}\lambda_T^v.
\label{eq:weighted-higher-order-diff}
\end{equation}
Hence,
\begin{align}
|\Delta_i|
\le{}&
\sum_{T\in\mathcal H_i^+}
\frac{(|T|-2)(n-|T|)}{2|T|n}|\lambda_T^v|
\nonumber\\
&+
\sum_{T\in\mathcal H_i^-}
\frac{|T|-2}{2n}|\lambda_T^v|.
\label{eq:weighted-higher-order-bound}
\end{align}
If $\mathcal H=\{T\subseteq N:3\le |T|\le n-1\}$, then the simpler bound
\begin{equation}
\left|\psi_i^{\mathrm{ESENSC}}(v)-\psi_i^{\mathrm{SHAP}}(v)\right|
\le
\frac{1}{2}\sum_{T\in\mathcal H}|\lambda_T^v|
\label{eq:simple-higher-order-bound}
\end{equation}
holds for every $i\in N$.
\end{prop}

\textbf{Proof.}
Equation \eqref{eq:weighted-higher-order-diff} follows by collecting the coefficients in \eqref{eq:esensc-shap-dividend}. The terms of order two vanish and the grand-coalition term cancels, leaving only coalitions of sizes $3,\ldots,n-1$. The first inequality follows from the triangle inequality. Both coefficients in \eqref{eq:weighted-higher-order-bound} are at most $1/2$, which yields \eqref{eq:simple-higher-order-bound}. \qed

The bound is necessarily scale-dependent: multiplying the entire game by a constant proportionally scales the Harsanyi dividends and both attribution vectors. Its implication is therefore relative rather than absolute---ESENSC must remain close to SHAP whenever the total absolute mass of the higher-order Harsanyi-dividend components associated with intermediate coalitions is small.

The assumption $\mathcal A(v) = N$ enters only through the denominator of the residual share. For general $v\in\mathbb V$ and $i \in \mathcal A(v)$, the same argument yields \eqref{eq:esensc-shap-dividend} with $1/n$ replaced by $1/|\mathcal A(v)|$. In that case the coefficient on $\lambda^v_N$ equals $(n-2)(|\mathcal A(v)|-n)/(2n|\mathcal A(v)|)$, which is nonzero whenever $|\mathcal A(v)| < n$. The grand-coalition dividend then no longer cancels, and the bound \eqref{eq:simple-higher-order-bound} need not hold. If every player outside $\mathcal A(v)$ is a null player, Propositions \ref{prop:harsanyi-difference} and \ref{prop:higher-order-bound} continue to hold with $N$ and $n$ replaced by $\mathcal A(v)$ and $|\mathcal A(v)|$.

\begin{coro}[Exact agreement on 2-additive games]\label{cor:two-additive}
For every 2-additive game $v\in\mathbb V$, ESENSC and exact SHAP coincide: $\psi^{\mathrm{ESENSC}}(v)=\psi^{\mathrm{SHAP}}(v)$.
\end{coro}

\textbf{Proof.}
In a 2-additive game, $\lambda_T^v=0$ for all $|T|\ge 3$. Hence
\begin{equation*}
b_i(v)=\lambda^v_{i}+\frac{1}{2}\sum_{\substack{T\ni i\\ |T|=2}}\lambda_T^v
=\psi_i^{\mathrm{SHAP}}(v).
\end{equation*}
Furthermore,
$\sum_{i\in N}b_i(v)=v(N)-v(\emptyset)$, so the ESENSC residual is zero. 
Thus the equality holds for both boundary-active and boundary-inactive players. \qed

\begin{coro}[Positivity on monotone 2-additive games]\label{monotone 2ad}
If $v\in\mathbb V$ is monotone (i.e., $v(S) \le v(T)$ whenever $S \subseteq T$) and 2-additive, then $\psi_i^{\mathrm{ESENSC}}(v)\ge 0$ for every $i\in N$.
\end{coro}

\textbf{Proof.}
By Corollary~\ref{cor:two-additive}, ESENSC coincides with the Shapley value. For a monotone game, every marginal contribution entering the Shapley value is nonnegative. Hence all ESENSC attributions are nonnegative. \qed

Positivity does not hold in general, even for monotone games. For example, let $N=\{1,2,3,4\}$, let $v(\emptyset)=0$, let $v(N)=v(\{2,3,4\})=1$, let $v(S)=0.1$ for every coalition containing player 1 except $N$, and let $v(S)=0$ otherwise. This game is monotone, yet ESENSC assigns player 1 the negative attribution $-0.05$.

\begin{prop}[Characterization of false-null attribution]\label{prop:false-null}
Suppose that $v \in \mathbb V$ and $i\notin\mathcal A(v)$, so that ESENSC assigns $\psi_i^{\mathrm{ESENSC}}(v)=0$. Then
\begin{equation}
\lambda^v_{i}=0, \qquad \sum_{T\subseteq N:\,i\in T}\lambda_T^v=0.
\label{eq:boundary-null-dividend}
\end{equation}
Exact SHAP nevertheless assigns a nonzero value to $i$ if and only if
\begin{equation}
\sum_{T\subseteq N:\,i\in T}\frac{\lambda_T^v}{|T|}\neq 0.
\label{eq:false-null-shap}
\end{equation}
Consequently, a false-null judgment relative to exact SHAP can occur only when nonzero dividends involving $i$ appear at different coalition cardinalities and cancel in the unweighted boundary sum in \eqref{eq:boundary-null-dividend}, but not under the Shapley weights in \eqref{eq:false-null-shap}.
\end{prop}

\textbf{Proof.}
The condition $i\notin\mathcal A(v)$ means $v(i)=v(\emptyset)$ and $v(N)=v(N\setminus i)$. The first equality gives $\lambda^v_{i}=0$, while the second gives $\sum_{T\ni i}\lambda_T^v=0$. Equation \eqref{eq:shap-harsanyi} then yields the stated condition for a nonzero SHAP value. If all nonzero dividends involving $i$ had the same cardinality, the unweighted and Shapley-weighted sums would differ only by a common positive factor, so the two cancellation conditions would coincide. \qed

In particular, when $n=3$, ESENSC and SHAP coincide if and only if $\mathcal{A}(v)=N$ or $\lambda^{v}_N=0$; hence, they differ whenever $\mathcal{A}(v) \neq N$ and $\lambda^{v}_N \neq 0$.

\endgroup

\subsection*{C. Derivation of PARPA}
In this appendix, we derive PARPA and discuss why PARPA avoids the order-reversal problem. 
First, we introduce a complementary proportional rule, RPA, which reverses the weighting structure of PA. While PA assigns weights $v(j)-v(\emptyset)$, RPA assigns $v(N)-v(\emptyset)-\sum_{i\neq j}(v(i)-v(\emptyset))$, the residual value that remains when the marginal contributions of all other features are deducted from the total surplus. Formally, 
\begin{eqnarray}
\psi^{\mathrm{RPA}}_{j}
&=&
\frac{
v(N)-v(\emptyset)-\sum_{i\neq j}\left\{v(i)-v(\emptyset)\right\}
}{
\sum_{i\in N}
\left\{
v(N)-v(\emptyset)-\sum_{k\neq i}\left\{v(k)-v(\emptyset)\right\}
\right\}
}
\cdot
\left\{v(N)-v(\emptyset)\right\} \notag
\\
&=&
\frac{
v(N)-v(\emptyset)-\sum_{i\neq j}\left\{v(i)-v(\emptyset)\right\}
}{
n\left\{v(N)-v(\emptyset)\right\}
-
(n-1)\sum_{i\in N}\left\{v(i)-v(\emptyset)\right\}
}
\cdot
\left\{v(N)-v(\emptyset)\right\}.
\label{RPAarrange}
\end{eqnarray}

In Example \ref{ex0} in Subsection \ref{prop-related}, we have $\psi^{\mathrm{RPA}}_{1}=-1/3$ and 
$\psi^{\mathrm{RPA}}_{2}=4/3$. Thus, RPA satisfies the order preservation property in this case.  
This suggests that when the signs of $v(N)-v(\emptyset)$ and $\sum_{i\in N}\{v(i)-v(\emptyset)\}$ differ, RPA may be preferable to PA.

We now generalize this argument. Assume $v(N)-v(\emptyset)\neq 0$, since otherwise attribution is trivial. From \eqref{PA} and \eqref{RPAarrange}, the order-reversal problem arises in PA when $(v(N)-v(\emptyset))/\left[\sum_{i\in N}\left\{v(i)-v(\emptyset)\right\}\right]<0,$ and arises in RPA when
$(v(N)-v(\emptyset))/\left[n\left\{v(N)-v(\emptyset)\right\}-(n-1)\sum_{i\in N}\left\{v(i)-v(\emptyset)\right\}\right]<0.$ In both cases, whether order reversal occurs depends solely on the sign combination of $v(N)-v(\emptyset)$ and $\sum_{i\in N}\{v(i)-v(\emptyset)\}$.

Table \ref{hikakutable} summarizes the cases, where the first (second) entry in each cell corresponds to PA (RPA). ``OK'' indicates no order reversal, ``X'' indicates it necessarily arises, ``?'' indicates it may arise, and ``NA'' indicates the rule is undefined. The table shows that for any sign configuration, at least one of PA or RPA avoids order reversal. Thus, by selecting PA or RPA appropriately, the order-reversal problem can be avoided.

\begin{table}[htbp]
\centering
\begin{tabular}{c|ccc}
\diagbox[width=18em]{$v(N)-v(\emptyset)$}{$\sum_{i\in N}\{v(i)-v(\emptyset)\}$}
 & Positive & Negative & Zero \\
\hline
Positive & (OK, ?) & (X, OK)  & (NA, OK)  \\
Negative & (X, OK) & (OK, ?) & (NA, OK)
\end{tabular}
\caption{PA and RPA: cases where the order-reversal problem does not arise (``OK'')}
\label{hikakutable}
\end{table}

A simple selection rule is to choose the rule that corresponds to the ``OK'' case:
\begin{equation*}
\psi^{\mathrm{PARPA}}_{j} =
\begin{cases}
\psi^{\mathrm{PA}}_{j}
&
\text{if }
\big[v(N)-v(\emptyset)\big]
\cdot
\big[\sum_{i\in N}\{v(i)-v(\emptyset)\}\big]
>0,
\\[6pt]
\psi^{\mathrm{RPA}}_{j}
&
\text{otherwise.}
\end{cases}
\label{PARAP_form}
\end{equation*}
	
In sum, the rule $\psi^{\mathrm{PARPA}}$ is a PA--type AFA to address the order-reversal problem and preserves the ranking induced by $v(j)-v(\emptyset)$.

\subsection*{D. Additional Experiments on the HELOC Dataset}
This appendix presents additional experiments using the HELOC dataset to further validate the findings in the main text. In this experiment, formulated as a regression task, the target variable is \textit{ExternalRiskEstimate}, a consolidated risk score derived from credit bureau data. The dataset contains 22 features in total; to evaluate the performance of AFA methods across various input dimensions, we select $n$ features, where $n \in \{6, 8, 10, 12, 14, 16\}$. In addition to the Neural Network and XGBoost models employed in the main text, we include Random Forest as an additional predictive model to ensure a broader evaluation. For the random forest model, we use \texttt{RandomForestRegressor} from scikit-learn and tune its hyperparameters via grid search with five-fold cross-validation. The search space includes the number of trees (50, 100, and 200) and the maximum tree depth (10, 20, or unrestricted). We select the model that minimizes cross-validated mean squared error. Consistent with the main experiments, we use the first 1,000 observations and compute attributions for the test split (see Appendix E). Each AFA method is assessed based on its deviation from the exact SHAP values, using the Normalized Attribution Distance (NAD) and the Mean Rank Correlation (MRC) defined in Eqs. \eqref{NAD} and \eqref{MRC}, respectively.

The experimental results are summarized in Figures \ref{chr_Appendix_dev_from_SHAP_01} and \ref{chr_Appendix_dev_from_SHAP_02}, which correspond to the NAD and MRC metrics. In each figure, the left, center, and right panels show the results for the Neural Network, XGBoost, and Random Forest models, respectively. The horizontal axis of each panel represents the number of features, while the vertical axis indicates the value of each deviation indicator. Note that the values for PAROP are displayed on the right-hand scale, as their magnitude differs significantly from those of other AFA methods.

\begin{figure}[tb]
    \centering
    \includegraphics[width=1\linewidth]{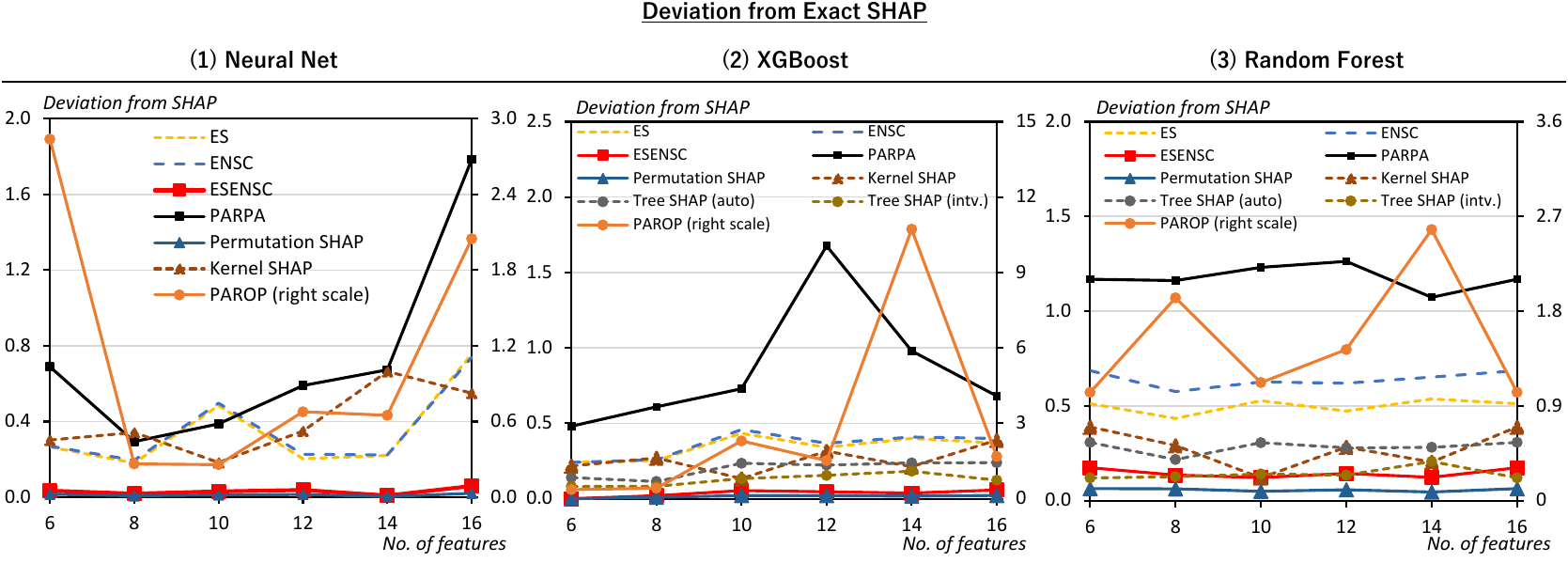}
    \caption{NAD values (deviation from exact SHAP) on the HELOC dataset. The three panels correspond to Neural Network (left), XGBoost (center), and Random Forest (right).}
    \label{chr_Appendix_dev_from_SHAP_01}
\end{figure}

\begin{figure}[tb]
    \centering
    \includegraphics[width=1\linewidth]{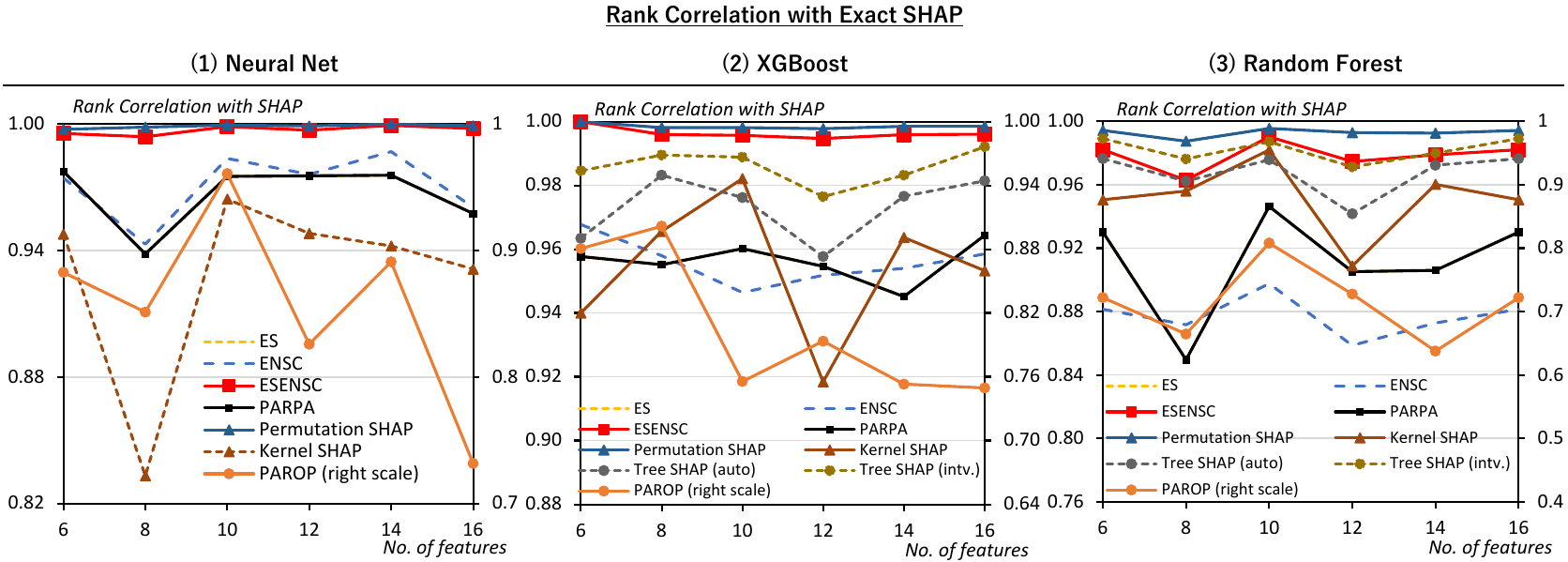}
    \caption{MRC values (rank correlation with exact SHAP) on the HELOC dataset. The three panels correspond to Neural Network (left), XGBoost (center), and Random Forest (right).}
    \label{chr_Appendix_dev_from_SHAP_02}
\end{figure}

These results strongly support the robustness of the findings presented in Section \ref{experiment}. Specifically, ESENSC consistently exhibits a minimal deviation from the exact SHAP values, with its rank correlation remaining remarkably close to 1. Furthermore, ESENSC generally outperforms existing SHAP approximation algorithms in terms of overall accuracy. In contrast, PA--type AFA methods show unstable performance; in particular, PAROP---which lacks the order preservation property---demonstrates significant instability. These patterns are similar across the datasets, predictive models, and feature dimensions considered, and are consistent with the theoretical analysis provided in the main text.

\subsection*{E. Model Description}
In Section \ref{experiment}, we adopt XGBoost and a neural network as the learning models.
In this appendix, we describe these models in detail.
 
For XGBoost, we use the \texttt{XGBRegressor} implementation and tune its hyperparameters via randomized search with three-fold cross-validation. The search space includes the number of estimators (100--1000), learning rate (uniformly sampled from 0 to 0.5), maximum tree depth (2 to 10), subsample ratio (0.5 to 1.0), and column sampling ratio (0.8--1.0). We perform $200$ iterations and select the model that minimizes cross-validated mean squared error. These ranges follow common practice in empirical applications and provide a flexible yet standard tuning setup. 
 
For the neural network model, we use \texttt{MLPRegressor} from scikit-learn within a preprocessing pipeline that includes feature standardization. Hyperparameters are tuned via grid search with five-fold cross-validation. We consider two hidden-layer architectures with moderate widths (ranging from $(64,64)$ to $(256,128)$), ReLU and tanh activation functions, L2 regularization parameters ($\alpha$) between $10^{-7}$ and $10^{-4}$, and initial learning rates of $10^{-3}$ and $3\times 10^{-4}$. Training is performed with a maximum of $5000$ iterations and early stopping based on a validation split. These choices follow common practice for multilayer perceptrons in tabular regression. For each dataset, we use the first 1,000 observations in the original order of the data. These are split into training (80\%) and test (20\%) sets, and attributions are computed for the test observations. A fixed random seed (42) is used for data generation, the train/test split, and model training.

Overall, both models are trained with cross-validated hyperparameter tuning over standard parameter ranges, ensuring that the comparison of attribution methods is conducted under conventional and well-established settings.

\subsection*{F. Proofs of Lemma \ref{lemma_01} and Corollaries \ref{corollary_01} and \ref{corollary_02}}

\textbf{Proof of Lemma \ref{lemma_01}:} 
The proof consists of three claims. 

\begin{claim}\label{claim_01}
Each of $\emptyset$ and $N$ is relevant to $\psi$ on $\mathbb{V}$.
\end{claim}

\textbf{Proof of Claim \ref{claim_01}:}
Let $v:2^N \to \mathbb{R}$ be $v(\emptyset)=0$ and $v(S)=1$ if $S \neq \emptyset$. 
Let $w:2^N \to \mathbb{R}$ be $w(\emptyset)=-1$ and $w(S)=v(S)$ if $S \neq \emptyset$. 
Let $u:2^N \to \mathbb{R}$ be $u(N)=2$ and $u(S)=v(S)$ if $S \neq N$. 
Now $v,w,u \in \mathbb{V}$. 

From efficiency, $\sum_{i \in N}\psi_i(v)=1$ and $\sum_{i \in N} \psi_i(w)=2$. 
Since $v$ and $w$ differ only in the worth of $\emptyset$, $\emptyset$ is relevant to $\psi$ on $\mathbb{V}$. 
Similarly, $\sum_{i \in N} \psi_i(u)=2$, and since $v$ and $u$ differ only in the worth of $N$, $N$ is relevant to $\psi$ on $\mathbb{V}$.\qed

If $n=1$, Claim \ref{claim_01} is sufficient. 
If $n=2$, every game is 2-additive. 
Hence, pairwise interaction splitting implies that $\psi_i(v)=\frac{1}{2}(v(i)-v(\emptyset)+v(N)-v(N \setminus i))$ for any $v \in \mathbb{V}$ and $i \in N$, which implies the desired result.  
In what follows, we consider the case when $n \ge 3$.  

\begin{claim}\label{claim_02}
Fix $i \in N$. 
Then, $\{i\}$ is relevant to $\psi$ on $\mathbb{V}$.
\end{claim}

\textbf{Proof of Claim \ref{claim_02}:}
Let $v:2^N \to \mathbb{R}$ be $v(\emptyset)=0$ and $v(S)=1$ if $S \neq \emptyset$. 
Let $w:2^N \to \mathbb{R}$ be $w(S)=|S|$ for any $S \subseteq N$. 
Let $u:2^N \to \mathbb{R}$ be $u(\emptyset)=u(i)=0$ and $u(S)=1$ if $S \neq \emptyset,\{i\}$.
Now $v,w,u \in \mathbb{V}$. 

Because $w$ is 2-additive, $\psi_k(w)=1$ for any $k \in N$. 
It is straightforward to verify that $v$, $w$ and any $j,k \in N$ satisfy conditions (i)-(iii) of restricted differential marginality. 
From restricted differential marginality, $\psi_k(v)-\psi_j(v)=\psi_k(w)-\psi_j(w)=0$ for any $k,j \in N$. 
From efficiency, $\psi_k(v)=\frac{v(N)-v(\emptyset)}{n}=\frac{1}{n}$ for any $k \in N$. 
Because $i$ is a null player in $u$, the null player property implies that $\psi_i(u)=0$. 
Since $v$ and $u$ differ only in the worth of $\{i\}$, while $\psi_i(v)=\frac{1}{n}\neq 0=\psi_i(u)$, $\{i\}$ is relevant to $\psi$ on $\mathbb{V}$. 
\qed

\begin{claim}\label{claim_03}
Fix $i \in N$. 
Then, $N \setminus i$ is relevant to $\psi$ on $\mathbb{V}$.
\end{claim}

\textbf{Proof of Claim \ref{claim_03}:}
Let $v:2^N \to \mathbb{R}$ be $v(N)=1$ and $v(S)=0$ if $S \neq N$. 
Let $w:2^N \to \mathbb{R}$ be $w(S)=|S|$ for any $S \subseteq N$. 
Let $u:2^N \to \mathbb{R}$ be $u(N)=u(N \setminus i)=1$ and $u(S)=0$ if $S \neq N, N \setminus \{i\}$.
Now $v,w,u \in \mathbb{V}$. 

Because $w$ is 2-additive, $\psi_k(w)=1$ for any $k \in N$. 
It is straightforward to verify that $v$, $w$ and any $j,k \in N$ satisfy conditions (i)-(iii) of restricted differential marginality. 
From restricted differential marginality, $\psi_k(v)-\psi_j(v)=\psi_k(w)-\psi_j(w)=0$ for any $k,j \in N$. 
From efficiency, $\psi_k(v)=\frac{v(N)-v(\emptyset)}{n}=\frac{1}{n}$ for any $k \in N$. 
Because $i$ is a null player in $u$, the null player property implies that $\psi_i(u)=0$. 
Since $v$ and $u$ differ only in the worth of $N \setminus \{i\}$, while $\psi_i(v)=\frac{1}{n}\neq 0=\psi_i(u)$, $N \setminus \{i\}$ is relevant to $\psi$ on $\mathbb{V}$. 
\qed

Note that Claims \ref{claim_02} and \ref{claim_03} hold for any $i \in N$. 
From Claims \ref{claim_01} to \ref{claim_03}, 
Lemma \ref{lemma_01} is obtained.\qed

\textbf{Proof of Corollary \ref{corollary_01}:} 
If $n \ge 3$, $|\mathcal{C}|=2n+2$.  
For $n=1,2,3$, $\mathcal{C}=2^N$, and hence the following conclusion is immediate. 
Suppose $n \ge 4$.   
Since information parsimony allows at most $2n+2$ relevant coalitions, Lemma \ref{lemma_01} implies that no coalition outside $\mathcal{C}$ is relevant to $\psi$. 
Let $v,w \in \mathbb{V}$ satisfy $v(S)=w(S)$ for every $S \in \mathcal{C}$, and enumerate the coalitions $2^N \setminus \mathcal{C}$ as $S_1, \dots, S_m$. 
Starting from $v$, construct a sequence of games $v^0, \dots v^m$, where $v^0=v$ and $v^k$ is obtained from $v^{k-1}$ by replacing the worth of $S_k$ with $w(S_k)$. 
Since membership in $\mathbb{V}$ depends only on the worths of coalitions in $\mathcal{C}$, every $v^k$ belongs to $\mathbb{V}$. 
Moreover, since $S_k$ is not relevant to $\psi$, we have $\psi(v^{k-1})=\psi(v^k)$ for every $k=1,\dots m$. 
Therefore, $\psi(v)=\psi(v^0)=\dots=\psi(v^m)=\psi(w)$. \qed

\textbf{Proof of Corollary \ref{corollary_02}:} 
Let $v \in \mathbb{V}$ satisfy the condition in Corollary \ref{corollary_02}. 
If $n \le 2$, every game is $2$-additive, and the result follows directly from pairwise interaction splitting. 
Suppose $n \ge 3$.
For each $i \in N$, let
$$
d_i=v(N)-v(N\setminus i)-v(i)+v(\emptyset).
$$
Construct a $2$-additive game $w$ whose Harsanyi dividends are given by 
$$
\lambda_\emptyset^w=v(\emptyset),\qquad
\lambda_i^w=v(i)-v(\emptyset),
$$
and, for distinct $i,j\in N$,
$$
\lambda_{\{i,j\}}^w
=
\frac{
d_i+d_j-\frac{1}{n-1}\sum_{k\in N}d_k
}{n-2}
$$
with all remaining Harsanyi dividends equal to zero. 
Then
$$
\sum_{j\neq i}\lambda_{\{i,j\}}^w=d_i
$$
for every $i\in N$. Hence,
$$
w(i)=v(i)
\quad\text{and}\quad
w(N)-w(N\setminus i)=v(N)-v(N\setminus i).
$$
Moreover,
$$
w(N)-w(\emptyset)
=
\sum_{i\in N}\lambda_i^w
+\frac12\sum_{i\in N}d_i
=
\sum_{i\in N}\frac12
\bigl(v(i)-v(\emptyset)+v(N)-v(N\setminus i)\bigr)
=
v(N)-v(\emptyset).
$$
Thus, $w(N)=v(N)$ and $w(N \setminus i)=v(N \setminus i)$, which imply $w \in \mathbb{V}$.  
By Corollary \ref{corollary_01}, $\psi(v)=\psi(w)$. 
Since $w$ is $2$-additive, pairwise interaction splitting implies
$$
\psi_i(v)
=
\psi_i(w)
=
\lambda_i^w+\frac12\sum_{j\neq i}\lambda_{\{i,j\}}^w
=
\frac12\bigl(v(i)-v(\emptyset)+v(N)-v(N\setminus i)\bigr)
$$
for every $i\in N$.\qed

\subsection*{G. Proof of Theorem \ref{maintheoremforaxiom}}
This appendix shows the proof of Theorem \ref{maintheoremforaxiom}. First, we prove the ``if'' part with the following Claim \ref{hajimeclaim}: 
\begin{claim}\label{hajimeclaim}
$\psi^{\mathrm{ESENSC}}$ satisfies efficiency, the null player property, restricted differential marginality, pairwise interaction splitting, and information parsimony.
\end{claim}
\textbf{Proof of Claim \ref{hajimeclaim}}: \ Efficiency and the null player property are apparent. 
Regarding restricted differential marginality, let $v,w \in \mathbb{V}$ and $i,j \in N$ satisfying (i),(ii), and (iii) in the definition of the axiom. 
Then, 
\begin{multline*}
\psi^{\mathrm{ESENSC}}_i(v)-\psi^{\mathrm{ESENSC}}_j(v)
=b_i(v)-b_j(v) =\frac{1}{2}(v(i)-v(j)-v(N \setminus i)+v(N \setminus j))\\
=\frac{1}{2}(w(i)-w(j)-w(N \setminus i)+w(N \setminus j)) = b_i(w)-b_j(w) = \psi^{\mathrm{ESENSC}}_i(w)-\psi^{\mathrm{ESENSC}}_j(w). 
\end{multline*}
Regarding information parsimony, $\psi^{ESENSC}$ depends only on coalition values in $\mathcal{C}$ and $|\mathcal{C}| \le 2n+2$ for every $n$. Pairwise interaction splitting follows directly from Corollary \ref{cor:two-additive} and its proof.
\qed 

We then prove the ``Only-if'' part with the following Claim \ref{claimsono2} and Claim \ref{claimsono3}. Let $\psi$ be an attribution rule satisfying the five axioms in Theorem \ref{maintheoremforaxiom}. 
\begin{claim}\label{claimsono2}
From the null player property and Corollary \ref{corollary_01}, for any $v \in \mathbb{V}$ and $i \in N$ with $v(i)=v(\emptyset)$ and $v(N)=v(N \setminus i)$, $\psi_i(v)=0$. 
\end{claim}
\textbf{Proof of Claim \ref{claimsono2}}: \ If $n=1$ or $2$, the null player property implies the fact. 
%
%
If $n \ge 4$, given $v \in \mathbb{V}$ and $i \in N$ with $v(i)=v(\emptyset)$ and $v(N)=v(N \setminus i)$, let $v^i:2^N \to \mathbb{R}$, where 
\begin{equation*}
v^i(S)=\begin{cases}
v(S) & \text{ if } S \ni i \text{ and } |S| \neq 2\\
v(S \setminus i) & \text{ if } S \ni i \text{ and } |S|=2\\
v(S) & \text{ if } S \not\ni i \text{ and } |S|=0,1,n-1\\
v(S \cup i ) & \text{ if } S \not\ni i \text{ and } |S| \neq 0,1,n-1.
\end{cases}
\end{equation*}
By definition,  
$v^i(\emptyset)=v(\emptyset)$, 
for any $j \in N$, $v^i(j)=v(j)$ and 
$v^i(N \setminus j)=v(N \setminus j)$,\footnote{Note that because $n \ge 4$, it holds that $n-1 \neq 2$.} and 
$v^i(N)=v(N)$.
Hence, $v^i \in \mathbb{V}$.  
From Corollary \ref{corollary_01}, 
\begin{equation}\label{viv}
\psi(v^i)=\psi(v). 
\end{equation}

Additionally, by definition, for any $S \subseteq N \setminus i$, 
\begin{equation*}
v^i(S \cup i)-v^i(S)
=\begin{cases}
v(i)-v(\emptyset)=0 & \text{ if } |S|=0,\\
v(S)-v(S)=0 & \text{ if } |S|=1,\\
v(S \cup i)-v(S \cup i)=0 & \text{ if } |S|=2,\dots,n-2,\\
v(N)-v(N \setminus i)=0 & \text{ if } |S|=n-1.
\end{cases}
\end{equation*}
Thus, $i$ is a null player in $v^i$. 
From the null player property, 
\begin{equation}\label{vinull}
\psi_i(v^i)=0.
\end{equation}
Combining (\ref{viv}) and (\ref{vinull}), we obtain the desired result. \qed
\begin{claim}\label{claimsono3}
If $\psi$ satisfies the five axioms in Theorem \ref{maintheoremforaxiom}, then $\psi = \psi^{\mathrm{ESENSC}}$. 
\end{claim}
\textbf{Proof of Claim \ref{claimsono3}}: \ We divide the proof into two cases. 

Case (i): $v \in \mathbb{V}$ with $v(N)-v(\emptyset)=\sum_{k \in N}\frac{1}{2}\left(v(k)-v(\emptyset)+v(N)-v(N \setminus k)\right)$. 
Equivalently, Case (i) is characterized by $R(v)=0$. Note that if $n=1$ or $2$, any $v \in \mathbb{V}$ belongs to this case. From Corollary \ref{corollary_02}, the desired result is obtained. 

Case (ii): $v \in \mathbb{V}$ with $v(N)-v(\emptyset) \neq \sum_{k \in N}\frac{1}{2}\left(v(k)-v(\emptyset)+v(N)-v(N \setminus k)\right)$. 
Note that this case occurs only when $n \ge 3$. 
Under the maintained assumption $n \neq 3$, we therefore have $n \ge 4$ throughout Case (ii). 
Note also that from the definition of $\mathbb{V}$, $\mathcal{A}(v)\neq \emptyset$. 

We further subdivide this case into two cases. 

Case (ii)-(a): $|\mathcal{A}(v)|=1$. From Claim \ref{claimsono2}, $\psi_{\ell}(v)=0=\psi^{\mathrm{ESENSC}}_{\ell}(v)$ for any $\ell \in N \setminus \mathcal{A}(v)$. From efficiency, for $i \in \mathcal{A}(v)$,
\begin{equation*}
\psi_i(v)=v(N)-v(\emptyset)= b_i(v)+\frac{v(N)-v(\emptyset)- b_i(v)}{1} =b_i(v)+\frac{R(v)}{|\mathcal{A}(v)|}= \psi^{\mathrm{ESENSC}}_i(v). 	
\end{equation*}

Case (ii)-(b): $|\mathcal{A}(v)| \ge 2$. Given $v \in \mathbb{V}$, let $\bar{v}:2^N \to \mathbb{R}$ where 
\begin{equation*}
\bar{v}(S)=
\begin{cases}
v(S) & |S|=0,\dots,n-2,\\
v(S)-R(v)
& |S|=n-1, n.
\end{cases} 
\end{equation*}
Because $R(v) \neq 0$, 
 $v \neq \bar{v}$. 
Additionally, because $n \ge 3$, for any $\ell \in N$, $\bar{v}(\ell)-\bar{v}(\emptyset)=v(\ell)-v(\emptyset)$ and $\bar{v}(N)-\bar{v}(N \setminus \ell)=v(N)-v(N \setminus \ell)$. 
Hence, $\bar{v} \in \mathbb{V}$, $b_{\ell}(\bar{v})=b_{\ell}(v)$ for any $\ell \in N$, and $\mathcal{A}(\bar{v})=\mathcal{A}(v)$. Moreover, for any $i,j \in \mathcal{A}(v)$ and any $S \subseteq N \setminus \{i,j\}$, $\bar{v}(S \cup i)-\bar{v}(S \cup j)=v(S \cup i)-v(S \cup j)$, since $S \cup i$ and $S \cup j$ have the same cardinality. Additionally, 
$$
R(\bar{v})=\bar{v}(N)-\bar{v}(\emptyset)-\sum_{i \in \mathcal{A}(\bar{v})}b_i(\bar{v}) =v(N)-R(v)-v(\emptyset)-\sum_{i \in \mathcal{A}(v)}b_i(v)=0,
$$
Hence, $\bar{v}$ belongs to Case (i). From Case (i), 
\begin{equation*}\label{barv}
\psi(\bar{v})=\psi^{\mathrm{ESENSC}}(\bar{v}). 
\end{equation*}

From restricted differential marginality, for any $i,j \in \mathcal{A}(v)$, 
\begin{multline*}
\psi_i(v)-\psi_j(v)=\psi_i(\bar{v})-\psi_j(\bar{v})=\psi^{\mathrm{ESENSC}}_i(\bar{v})-\psi^{\mathrm{ESENSC}}_j(\bar{v}) = b_i(\bar{v})-b_j(\bar{v})\\
=\frac{1}{2}\left( \bar{v}(i)-\bar{v}(j)-\bar{v}(N \setminus i)+\bar{v}(N \setminus j)\right)
=\frac{1}{2}\left( v(i)-v(j)-v(N \setminus i)+v(N \setminus j)\right)\\
=b_i(v)-b_j(v), 
\end{multline*}
which implies that 
\begin{equation}\label{ij}
\psi_i(v)-b_i(v)=\psi_j(v)-b_j(v).
\end{equation}
Hence, every $i \in \mathcal{A}(v)$ receives the same adjustment, say $c$.  
From Claim \ref{claimsono2}, 
\begin{equation}\label{null}
\psi_{\ell}(v)=0=\psi^{\mathrm{ESENSC}}_{\ell}(v) \text{ for any } \ell \in N \setminus \mathcal{A}(v).
\end{equation} 
From efficiency and equation (\ref{null}), 
\begin{equation}\label{eff}
\sum_{j \in \mathcal{A}(v)}\psi_{j}(v)=v(N)-v(\emptyset).
\end{equation}
From equations (\ref{ij}) and (\ref{eff}), 
$$
|\mathcal{A}(v)|c=v(N)-v(\emptyset)-\sum_{i\in \mathcal{A}(v)}b_i(v)=R(v),
$$
and thus, 
$$
c=\frac{R(v)}{|\mathcal{A}(v)|}.  
$$ 
Therefore, for any $i \in \mathcal{A}(v)$,
$$
\psi_i(v)=b_i(v)+c=\psi^{\mathrm{ESENSC}}_i(v).
$$
Together with equation (\ref{null}), $\psi(v)=\psi^{\mathrm{ESENSC}}(v)$. 
\qed
\ \\ \textbf{(The end of the proof of Theorem \ref{maintheoremforaxiom}.)} 

The restriction of the axiomatization domain to $\mathbb{V}$ in Theorem \ref{maintheoremforaxiom}, as compared with Remark \ref{sankocasajus} in the subsequent Appendix I, is justified by Remark \ref{rem_nnoteq3} below. 

\begin{rem}\label{rem_nnoteq3}
Suppose that the domain of an attribution rule is enlarged to the class of all TU games and that the five axioms in Theorem \ref{maintheoremforaxiom} are imposed on this enlarged domain, with all quantifiers and relevance interpreted relative to the enlarged domain. 
Then, for $n \ge 4$, there exists no such attribution rule satisfying all five axioms.  
\end{rem}

\textbf{Proof of Remark \ref{rem_nnoteq3}.} 
Since all games used in the proof of Lemma \ref{lemma_01} belong to $\mathbb{V}$, the same argument shows that every coalition in $\mathcal{C}$ is relevant on the unrestricted domain. 
For $n \ge 4$, information parsimony implies that no coalition outside $\mathcal{C}$ is relevant. 
Changing the worths of coalitions outside $\mathcal{C}$ one at a time therefore yields the analogue of Corollary \ref{corollary_01} on the unrestricted domain.  
Take $v \not \in \mathbb{V}$ with $v(N) \neq v(\emptyset)$, $v(i)=v(\emptyset)$ and $v(N)=v(N \setminus i)$ for every $i \in N$. 
Applying the construction used in the proof of Claim \ref{claimsono2}, for each $i \in N$ there exists $v^i$ that agrees with $v$ on $\mathcal{C}$ and in which $i$ is a null player. 
Hence, $\psi_i(v)=\psi_i(v^i)=0$ for every $i \in N$, which contradicts efficiency. \qed

\subsection*{H. Independence of axioms in Theorem \ref{maintheoremforaxiom} }
In this appendix we demonstrate that each axiom in Theorem \ref{maintheoremforaxiom} is logically independent of the others for $n \ge 4$. 
To this end, we construct a series of alternative attribution rules, each of which satisfies all axioms except one. 
These counterexamples show that none of the axioms can be derived from the remaining ones, and hence each plays an essential role in the axiomatic characterization of $\psi^{\mathrm{ESENSC}}$. 
Establishing such independence strengthens the robustness of the characterization by confirming that the axiomatic system is minimal.

For any $v \in \mathbb{V}$ and any $i \in N$,
\begin{itemize}
\item let $\psi^1_i(v)=b_i(v)$,  
\item let $\psi^2_i(v)=b_i(v)+\frac{1}{n}R(v)$, 
\item let 
\begin{equation*}
\ \ \ \ \ \ \ \ \ \ \ \ \ \ \ \ \ \ \ \ \ \psi^3_i(v)=
\begin{cases}
0 \ \ \ \ \ \ \ \ \ \ \ \ \ \ \ \ \ \ \ \ \ \ \ \ \ \ \ \ \ \ \ \ \ \ \ \ \ \ \ \ \ \text{ if } i \not\in \mathcal{A}(v),\\ 
b_i(v)+\frac{i}{\sum_{j \in \mathcal{A}(v)} j}R(v) \ \ \ \ \ \ \  \text{otherwise,} 
\end{cases}
\end{equation*}
\item let 
\begin{equation*}
\psi^4_i(v)=
\begin{cases}
0  \ \ \ \ \ \ \ \ \ \ \ \ \ \ \ \ \ \ \ \ \text{ if } i \not\in \mathcal{A}(v),\\
\frac{v(N)-v(\emptyset)}{|\mathcal{A}(v)|} \ \ \ \ \ \ \ \text{otherwise,} 
\end{cases}
\end{equation*}
and  
\item $\psi^5_i(v)=\psi^{\mathrm{SHAP}}_i(v)$.
\end{itemize}
$\psi^1$ satisfies all axioms in Theorem \ref{maintheoremforaxiom}  but efficiency. 
$\psi^2$ satisfies all axioms in Theorem \ref{maintheoremforaxiom}  but the null player property. 
$\psi^3$ satisfies all axioms in Theorem \ref{maintheoremforaxiom}  but restricted differential marginality. 
$\psi^4$ satisfies all axioms in Theorem \ref{maintheoremforaxiom}  but pairwise interaction splitting. 
$\psi^5$ satisfies all axioms in Theorem \ref{maintheoremforaxiom}  but information parsimony (for $n \ge 4$).

\subsection*{I. An existing axiomatic characterization of the Shapley value in classical TU games}\label{donyuu}
In this appendix, we briefly review existing axiomatizations of the Shapley value that are closely related to our axiomatization of ESENSC. To relate our framework to the classical cooperative-game-theoretic literature, we restrict attention to the standard domain of TU games in which the value of the empty coalition is normalized to zero:
$\bar{\mathbb{V}}=\{\, v:2^N\to\mathbb{R}
\mid
v(\emptyset)=0
\,\}.$\footnote{The result is considered to remain valid even when $v(\emptyset)\neq 0$. Since the cited paper develops the arguments under the normalization $v(\emptyset)=0$, we adopt the same convention here for consistency.}

\begin{rem}[Corollary 5 in \cite{Casajus11}]\label{sankocasajus}
On $\bar{\mathbb{V}}$, an attribution rule $\psi$ satisfies the following efficiency, the null player property, and differential marginality if and only if $\psi=\psi^{SHAP}$.
\begin{itemize}[leftmargin=*]
\item \textbf{Efficiency.}  
For any $v \in \bar{\mathbb{V}}$,  
$\sum_{i \in N}\psi_i(v)=v(N)-v(\emptyset).$
\item \textbf{Null player property.}  
If $i \in N$ satisfies $v(S \cup i)=v(S)$ for any $S \subseteq N \setminus i$, then $\psi_i(v)=0$.
\item \textbf{Differential marginality.}  
Let $v,w \in \bar{\mathbb{V}}$ and $i,j \in N$ satisfy  
$v(S \cup i)-v(S \cup j)=w(S \cup i)-w(S \cup j)$ for any $S \subseteq N \setminus \{i,j\}$. 
Then $\psi_i(v)-\psi_j(v)=\psi_i(w)-\psi_j(w).$
\end{itemize}
\end{rem}

It is well known that differential marginality is equivalent to fairness in the sense of van den Brink \cite{Brink_2002}; see also Proposition 3 in \cite{Casajus11}. Moreover, fairness (and therefore differential marginality) can be derived from the combination of two well-known axioms, namely symmetry and additivity (Proposition 2.4 in \cite{Brink_2002}). Thus, differential marginality is not an entirely independent requirement, but can be understood as a consequence of these more primitive and widely accepted axioms.

\end{document}